\documentclass{article}

\usepackage[preprint,nonatbib]{neurips_2026}

\usepackage[utf8]{inputenc}
\usepackage[T1]{fontenc}
\usepackage{hyperref}
\hypersetup{
  pdftitle={Structured Memory for Edge Language Models: Persistent Context and Corpus Retrieval via O(1) SSM State Injection},
  pdfauthor={Anusha Madan Gopal, Aras Pirbadian, Kristofor D. Carlson, M Anthony Lewis, Jonathan Tapson},
  pdfsubject={State-Space Models, Retrieval-Augmented Generation, Edge Inference}
}
\usepackage{url}
\usepackage{booktabs}
\usepackage{amsfonts}
\usepackage{amsmath}
\usepackage{amssymb}
\usepackage{amsthm}
\usepackage{nicefrac}
\usepackage{microtype}
\usepackage{xcolor}
\usepackage{graphicx}
\usepackage{enumitem}
\usepackage{placeins}

\newtheorem{theorem}{Theorem}
\newtheorem{proposition}{Proposition}

\title{Structured Memory for Edge Language Models: Persistent Context and Corpus Retrieval via O(1) SSM State Injection}

\author{%
  Anusha Madan Gopal \quad Aras Pirbadian \quad Kristofor D. Carlson \\
  \textbf{M Anthony Lewis} \quad \textbf{Jonathan Tapson} \\
  BrainChip Inc. \\
  23041 Avenida de la Carlota, Laguna Hills, CA \\
  \texttt{agopal@brainchip.com, apirbadian@brainchip.com, kcarlson@brainchip.com,} \\
  \texttt{tlewis@brainchip.com, jtapson@brainchip.com} \\
}

\begin{document}

\maketitle
\begin{abstract}
Retrieval-augmented generation (RAG) imposes a prefill cost proportional to retrieved context length, and---with Transformer backbones---a KV-cache that grows with each generated token. State-Space Models (SSMs) avoid the second cost by construction; we eliminate the first, collapsing prefill from $O(L_{\text{context}})$ to $O(1)$ per query. We introduce \textbf{PRECOG (Pre-Computed Context Injection)}, a retrieval mechanism that exploits a property unique to SSMs: the fixed-size, position-agnostic recurrent hidden state is a complete summary of everything the model has read. PRECOG pre-encodes document corpora offline as SSM hidden states and injects the best-matching state directly at query time, bypassing in-context re-ingestion entirely. The same state-injection mechanism enables \textbf{SMC (Structured Memory Consolidation)}: a hierarchical persistent memory with cognitive-domain clustering, an adjustable fidelity-vs-storage dial, and $O(1)$ session initialization, which consolidates short-term episodic states into long-term semantic memory and fuses both with retrieved corpus states at query time. We demonstrate the system on \textbf{TENNs-LLM}, a 1.2B-parameter gated-SSM language model with a 192~KB hidden state. PRECOG matches in-context RAG answer quality, reducing prefill latency from $\sim$27\,s to $<$6\,ms on edge hardware---a $\sim$4500$\times$ speedup that crosses the threshold from unusable to interactive. The mechanism is architecturally impossible for Transformer KV-caches, which are position-entangled and grow linearly with context length.
\end{abstract}

\section{Introduction}
\label{sec:intro}

Two structural limitations of the dominant Transformer architecture shape the practical deployment of language models. First, the key-value (KV) cache grows linearly with context length: each generated token must attend over all prior tokens, imposing memory bandwidth that scales with the sequence. Second, this cache is position-entangled---keys and values incorporate rotary or absolute positional encodings---which couples tokens to their absolute positions and precludes arbitrary reuse across contexts. Together these properties make Transformer decoders inefficient for long contexts and rigid with respect to cache reuse.

State-Space Models (SSMs) \cite{s4,mamba} compress prior context into a \emph{fixed-size}, \emph{position-agnostic} recurrent hidden state. Per-token inference cost is $O(1)$ in memory and $O(d \cdot N)$ in compute independent of context length; the state encodes \emph{what} has been read, not \emph{where} it was read. These properties have been studied primarily as efficiency gains, but they also enable a refactoring of retrieval-augmented generation (RAG) that is structurally inaccessible to Transformers: pre-computing retrieved corpora as recurrent states and injecting them directly at query time, eliminating context-token ingestion at prefill entirely.

We introduce \textbf{PRECOG (Pre-Computed Context Injection)}, a retrieval mechanism that reduces the context-ingestion cost at prefill from $O(L_{\text{context}})$ to $O(1)$. Offline, PRECOG pre-encodes each chunk of a document corpus by running an SSM language model over it and capturing the resulting per-layer hidden state. At query time, it retrieves the best-matching state via a lightweight embedding index and injects it directly into the model's recurrent state as an initial condition; the model then processes only the user query, with retrieved context already encoded in the initial state. The injection is exact: the SSM recurrence is time-translation invariant, so a state pre-computed by running the model over a chunk is identical to the state the model would reach by running over that chunk at the start of the query. We prove this formally as Theorem~\ref{thm:precog}: PRECOG and in-context RAG produce identical state trajectories under autoregressive SSM dynamics, so the empirical claim that PRECOG matches in-context RAG quality is a mathematical guarantee, not a hopeful experimental finding. This mechanism is the subject of a related pending patent application \cite{brainchip_patent}.

The same refactoring fails for self-attention with position encoding: the KV-cache is not time-translation invariant under rotary or absolute encodings, so a cache pre-computed at positions $0, \ldots, L$ is invalid at any other position. Recomputing it negates any prefill savings. Even setting position-entanglement aside, a Transformer KV-cache at matched scale requires $\sim$500$\times$ more storage per chunk (3000-token context) than the SSM hidden state (Section~\ref{sec:complexity}).

We instantiate PRECOG on \textbf{TENNs-LLM}, a 1.2B-parameter gated-SSM language model with 24 layers and a compact 192~KB total hidden state---small enough to make state-level retrieval and caching practical at corpus scale. We chose this backbone because of its integration with our target neuromorphic edge platform (Appendix~\ref{app:hardware}), where the on-device inference budget makes state-level retrieval operationally necessary. PRECOG itself is general: the time-translation-invariance assumption of Theorem~\ref{thm:precog} holds for any selective-SSM backbone, including Mamba~\cite{mamba}, Mamba-2~\cite{ssmduality}, and related linear-recurrent architectures~\cite{rwkv,retnet}.

On domain-specific question answering, PRECOG matches the answer quality of in-context RAG on the same backbone while reducing prefill latency from $\sim$27\,s to $<$6\,ms on the deployment target---a $\sim$4500$\times$ speedup. We further report deployment of the full system on a neuromorphic edge processor (Appendix~\ref{app:hardware}) where ingestion-free retrieval is operationally necessary.

The same state-injection mechanism extends beyond corpus retrieval. On-device language models often need to maintain persistent context that accumulates over time---a user's preferences, a device's interaction history, an appliance's accumulated logs---in a memory budget bounded by the deployment platform. We introduce \textbf{SMC (Structured Memory Consolidation)}, an organization of hidden states from past interactions into a hierarchical persistent memory: stored states are partitioned into cognitive-domain clusters with two-level retrieval (first to a domain, then to specific entries within it), an adjustable fidelity-vs-storage dial trades memory size for recall precision, and session initialization remains $O(1)$ regardless of how much context has accumulated. SMC consolidates short-term episodic states into long-term semantic memory and fuses both with retrieved corpus states at query time, unifying corpus retrieval and persistent memory under a single state-injection substrate. Section~\ref{sec:smc} describes the architecture and reports its empirical behavior.

\section{Related Work}

\paragraph{State-Space Models.} S4 \cite{s4}, S5 \cite{s5}, and Mamba \cite{mamba,ssmduality} established structured and selective SSMs as Transformer-competitive sequence models; RWKV \cite{rwkv} and RetNet \cite{retnet} explore related linear-recurrent designs. We use TENNs-LLM (Section~\ref{sec:tenns}), a selective SSM with bottlenecked gating; PRECOG is independent of these architectural details.

\paragraph{Retrieval-Augmented Generation.} RAG \cite{rag}, Fusion-in-Decoder \cite{fid}, RETRO \cite{retro}, Atlas \cite{atlas}, and REPLUG \cite{replug} all condition generation on retrieved text but require in-context ingestion at inference. xRAG \cite{xrag} compresses documents to reduce token cost but still ingests a compressed representation per query. Prompt-compression methods (LLMLingua \cite{llmlingua}, AutoCompressor \cite{autocompressor}, Gisting \cite{gisting}) similarly reduce but do not eliminate ingestion. PRECOG eliminates ingestion entirely by injecting the model's own hidden state.

Most directly related, State Soup~\cite{statesoup} and PICASO~\cite{picaso} cache and linearly compose SSM hidden states—State Soup for in-context skill mixing, PICASO via a permutation-invariant composition algebra for multi-context retrieval. Both methods are training-time modifications: PICASO trains a learned composition function over hidden states, and State Soup learns context-mixing weights. PRECOG, by contrast, operates on a stock SSM fine-tuned for standard next-token prediction; the state-injection identity (Theorem~\ref{thm:precog}) is purely algebraic and requires no PRECOG-specific training. PRECOG also addresses two settings PICASO and State Soup do not: (i) corpus retrieval with retrieval-score-weighted top-$k$ composition for RAG (Section~\ref{sec:precog-method}), and (ii) structured device memory, where a long-term persistent state---accumulated user, device, or appliance history---is fused with a short-term episodic state to answer queries (Section~\ref{sec:smc}). Both settings target the edge-deployment regime, where context-ingestion latency dominates total query latency on bandwidth-constrained hardware---a regime not analyzed by prior work, which is evaluated on data-center GPUs. Memory Caching \cite{memorycaching} caches state checkpoints \emph{within} a sequence to extend effective context during inference, while PRECOG caches states \emph{offline at indexing time}.

\paragraph{Hidden-state interventions and KV-cache reuse.} Soft prompts \cite{lester2021}, prefix tuning \cite{li2021prefix}, activation steering \cite{turner2023}, and in-context vectors \cite{liu2023icv} inject \emph{learned} vectors into language-model representations to condition generation; PRECOG injects \emph{model-derived} states from the retrieved chunk itself, with exact rather than approximate equivalence to in-context conditioning (Theorem~\ref{thm:precog}). On the Transformer side, paged attention \cite{vllm} and prompt caching reuse KV-caches across requests sharing a prefix, but the position-dependent structure of attention precludes arbitrary cross-query injection.

\paragraph{Edge inference.} Prior work on small efficient LMs \cite{phi,mobilellm} targets Transformer architectures. Our deployment (Appendix~\ref{app:hardware}) complements this literature; our primary contribution is algorithmic.

\section{TENNs-LLM}
\label{sec:tenns}

We instantiate PRECOG (corpus retrieval) and SMC (persistent device memory) on TENNs-LLM, a 1.2B-parameter decoder-only language model in the family of gated selective SSMs. The model has 24 layers with embedding dimension $d=2048$ and per-layer state dimension $N=4096$; full configuration is summarized in Table~\ref{tab:config}. Each TENNs block applies a selective SSM in which the discretization timescale $\Delta t$ and input projection $B$ are derived from the current token through a two-layer bottleneck (intermediate dimension 256), preserving Mamba-style temporal selectivity at reduced gating-projection parameter cost. At inference, the per-layer hidden state occupies 8~KB at FP16; across 24 layers the total recurrent state footprint is \textbf{192~KB}---small enough to make state caching, retrieval, and injection practical at corpus scale, and small enough to fit on-device alongside model weights on neuromorphic hardware (Appendix~\ref{app:hardware}).

\begin{table}[t]
  \caption{TENNs-LLM configuration hyperparameters.}
  \label{tab:config}
  \centering
  \small
  \begin{tabular}{ll}
    \toprule
    \textbf{Hyperparameter} & \textbf{Value} \\
    \midrule
    Embedding dimension & 2{,}048 \\
    Inner dimension & 4{,}096 \\
    Number of layers & 24 \\
    SSM state size & 4{,}096 (16 coeff $\times$ 256 repeat) \\
    SSM mode & Gated / selective ($\Delta t$, $B$ input-dependent) \\
    Gating bottleneck dim & 256 \\
    Causal conv kernel & 4 \\
    LoRA rank & 32 \\
    Tokenizer & Mistral-7B-v0.1 (32{,}000 vocab) \\
    Precision (training / inference) & FP32 / INT4 weight quantization \\
    Total parameters & $\sim$1.2B \\
    \bottomrule
  \end{tabular}
\end{table}

\paragraph{The recurrence is autoregressive and position-agnostic.} Each SSMLayer evolves a hidden state $\bar{h}_t \in \mathbb{C}^N$ via
\begin{equation}
\bar{h}_t = \bar{h}_{t-1} \cdot \exp(-\Delta t_t \cdot A) + B_t \cdot x_t \cdot \Delta t_t, \qquad y_t = C(\bar{h}_t),
\label{eq:ssm-state}
\end{equation}
where $\Delta t_t$ and $B_t$ are functions of the current token only and $A$ is a diagonal complex matrix with $A_k = -\text{softplus}(\alpha_k) + i\pi k$. The update map $\Phi(h, x) := h \cdot \exp(-\Delta t(x) \cdot A) + B(x) \cdot x \cdot \Delta t(x)$ depends on $(h, x)$ only, with no explicit dependence on $t$. This time-translation invariance is the property that makes pre-computed state injection exact (Theorem~\ref{thm:precog}); the bottlenecked gating, while parameter-efficient, is incidental to it. PRECOG and SMC therefore extend without modification to other position-agnostic selective SSMs, including Mamba~\cite{mamba} and Mamba-2~\cite{ssmduality}.

Full architectural details---TENNs block structure (RMSNorm, causal convolution front-end, gated residual path, output projection), $A$-matrix parameterization and timescale initialization, and training/inference duality (FFT-based parallel training in $O(L \log L)$ versus pure recurrent inference in $O(N)$ per token)---are in Appendix~\ref{app:training}.

\section{PRECOG: Pre-Computed Context Injection}
\label{sec:precog}

\begin{figure}[t]
  \centering
  \includegraphics[width=0.6\linewidth]{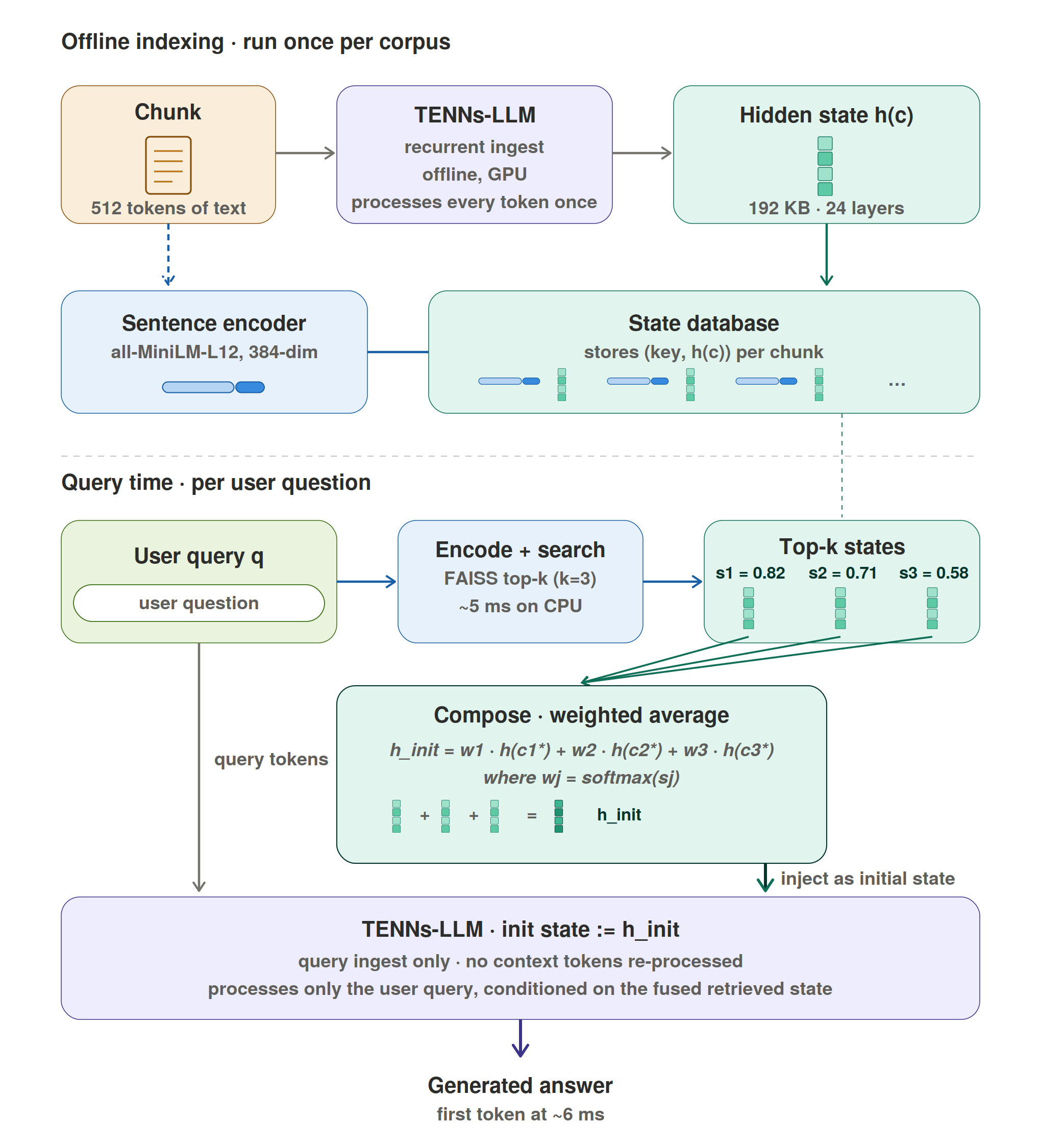}
  \caption{The PRECOG pipeline. \textbf{Offline indexing (top):} each chunk is encoded once by the SSM into a hidden state $h(c)$, paired with a sentence-encoder key. \textbf{Query time (bottom):} the query is encoded; the top-$k$ states ($k{=}3$ default) are retrieved by similarity, composed via softmax-weighted averaging into $h_{\text{init}}$, and injected as the initial recurrent state. The model processes only query tokens, producing the first generated token at $\sim$6\,ms after retrieval. No context tokens are re-ingested.}
  \label{fig:pipeline}
\end{figure}

\subsection{Motivation}
\label{sec:precog-motivation}

In conventional SSM-based RAG, retrieved chunks are re-ingested token-by-token at every query during the prefill phase, costing $O(L_{\text{context}})$ in latency and energy. For a 1.2B-parameter model at edge throughput (Appendix~\ref{app:hardware}), ingesting a single 512-token chunk takes $\sim$27 seconds before the first response token---a regime in which prefill dominates end-to-end latency. PRECOG exploits a property unique to recurrent models: the hidden state $\bar{h}$ after processing a chunk is a complete fixed-size summary of what the model read. Saving $\bar{h}$ and re-loading it as an initial condition is equivalent to re-ingesting the chunk. For SSMs this equivalence is \emph{exact}; context ingestion at prefill therefore reduces to a single state copy---$O(1)$ in retrieved-context length.

\subsection{Method}
\label{sec:precog-method}

\paragraph{Offline indexing.} Given a knowledge base $\mathcal{K}$, we partition it into chunks $\{c_i\}_{i=1}^M$. Chunk boundaries can follow any strategy---fixed-length token windows, paragraph or sentence boundaries, or document-aware splits---since the SSM consumes each chunk independently and produces a fixed-size hidden state regardless of input length. Throughout this paper we report results with $L_{\text{chunk}} = 512$ tokens for direct comparison with KV-cache baselines, but the method imposes no length constraint. Each chunk is run through the SSM in inference mode, capturing the 24-layer final hidden state $\bar{h}^{(i)} \in \mathbb{R}^{24 \times 4096}$ stored at FP16 (192~KB per entry, independent of chunk length). A lightweight sentence encoder $\phi$ (all-MiniLM-L12-v2, 384-dim) produces the retrieval key $k_i = \phi(c_i)$. Database entries $(k_i, \bar{h}^{(i)})$ are persisted to flash storage; on edge devices the corpus is too large to keep resident in DRAM (a $10{,}000$-chunk corpus is $1.9$~GB of states alone). Keys (compact 384-dim vectors, $\sim$3.8~MB for 10K chunks) are loaded into DRAM and indexed in FAISS for nearest-neighbor search; full hidden states remain on flash and are demand-loaded only when retrieved.

\paragraph{Query-time injection.} At query time, the system encodes the query $q$ via $\phi$, retrieves the top-$k$ chunks by cosine similarity ($k=3$ default), and demand-loads the corresponding states from flash into DRAM. The SSM's recurrent state is then initialized from the top-1 stored hidden state:
\begin{equation}
\bar{h}^{\text{init}}_\ell \leftarrow \bar{h}^{(i^*)}_\ell, \quad \ell = 1, \ldots, 24.
\end{equation}
The model processes the query tokens from this contextualized state and generates via top-$p$ sampling ($p=0.9$). The end-to-end query-time cost decomposes as: sentence-encoder forward pass over the query ($\sim$5~ms on CPU), FAISS top-$k$ search in DRAM ($<1$~ms for our corpus), flash-to-DRAM transfer of the 192~KB selected state ($\sim$50~$\mu$s on UFS~4.0 at 4.2~GB/s), and 24 vector copies into the SSM's recurrent buffer ($<1$~ms). PRECOG adds $\sim$6~ms total overhead versus zero-context generation, replacing the $\sim$27~s of in-context ingestion at the same throughput.

\subsection{Theoretical Guarantee}
\label{sec:theory}

The SSM update map can be written abstractly as $\Phi(h, x) := h \odot \alpha(x) + \beta(x)$, where $\alpha(x), \beta(x)$ depend on the current token only---there is no explicit dependence on position $t$. Let $\mathcal{S}(h, x_{1:T})$ denote the state after rolling the recurrence $T$ steps from $h$. The following identity is the load-bearing claim of the paper.

\begin{theorem}[PRECOG--RAG equivalence]
\label{thm:precog}
For any initial state $h_0$, context $c$, and query $q$,
\begin{equation}
\mathcal{S}(h_0,\; c \oplus q) \;=\; \mathcal{S}\bigl(\mathcal{S}(h_0, c),\; q\bigr),
\label{eq:precog-identity}
\end{equation}
where $\oplus$ denotes concatenation. The output logits at every position of $q$ are identical under both computations.
\end{theorem}

\begin{proof}[Proof sketch]
By induction on $|q|$. The recurrence is time-translation invariant: $\Phi$ depends on $(h, x)$ only, with no explicit position dependence. Full proof and a sufficient-statistic lemma in Appendix~\ref{app:proofs}.
\end{proof}

\paragraph{Interpretation.} Theorem~\ref{thm:precog} states that PRECOG is not an approximation to in-context RAG---it is the same computation, algebraically refactored. Pre-encoding a chunk into $\bar{h}^{(c)} := \mathcal{S}(h_0, c)$ produces bit-identical state trajectories (modulo FP16 quantization, bounded by $\sim 2^{-10}\|\bar{h}\|$ per element). The empirical claim ``PRECOG matches in-context RAG'' is therefore guaranteed by construction; deviations larger than the quantization bound indicate implementation issues, not method failure.

\paragraph{Memory horizon.} The closed-form unrolling of Eq.~\eqref{eq:precog-identity} (Appendix~\ref{app:memory}) shows that the contribution of context token $c_t$ to $\bar{h}^{(c)}$ decays as $\prod_{s=t+1}^{L} \alpha(c_s)$. Tokens beyond effective memory length $L_{\text{mem}}$ contribute exponentially less to the final state. Crucially, this forgetting is \emph{not specific to PRECOG}: it is the same forgetting that limits in-context RAG with the same backbone. Theorem~\ref{thm:precog} guarantees PRECOG inherits exactly the model's existing memory profile---neither gaining nor losing information relative to in-context ingestion. We measure $L_{\text{mem}}$ empirically and use it to motivate chunk-length ablations in Appendix~\ref{app:ablations}.

\paragraph{Why the identity fails for Transformers.} Under rotary position encoding, the per-token cache update is $K_t = R(t) W_K x_t$, where $R(t)$ is a position-dependent rotation. The map $x_{1:T} \mapsto \{(K_t, V_t)\}$ is therefore \emph{not} time-translation invariant: a cache pre-computed at positions $0\ldots L{-}1$ is invalid at positions $\tau, \ldots, \tau{+}L{-}1$ for $\tau \neq 0$. Recomputing the cache with corrected positions is equivalent to re-ingesting the chunk, negating any prefill savings.

\subsection{Complexity and Comparison to KV-Cache RAG}
\label{sec:complexity}

\paragraph{Complexity.} Context ingestion at prefill is $O(1)$ in retrieved-context length. The remaining prefill work is the sentence-encoder forward pass over the query, $O(L_{\text{query}})$ but independent of retrieved-context size. Retrieval over $M$ keys is $O(\log M)$ with an approximate index, the same as in-context RAG. Generation is $O(N)$ per token in the SSM state size, unchanged from standard inference.

\paragraph{Storage and latency vs. KV-cache RAG.} Pre-computed Transformer KV-caches are doubly impractical: they are position-entangled (Section~\ref{sec:theory}) \emph{and} their per-chunk storage scales as $O(L)$. For a 24-layer model with $d_{\text{head}}=128$, $n_{\text{heads}}=16$, $L=512$, a single chunk requires $\sim$16~MB at FP16 versus 192~KB for PRECOG---an 85$\times$ premium per chunk that compounds with corpus size. On flash, this gap means a 10K-chunk corpus consumes 1.9~GB for PRECOG versus 160~GB for hypothetical KV-cache storage---tractable on a phone in the first case, infeasible in the second. Table~\ref{tab:latency_memory} summarizes the full comparison.

\begin{figure}[t]
  \centering
  \includegraphics[width=0.75\linewidth]{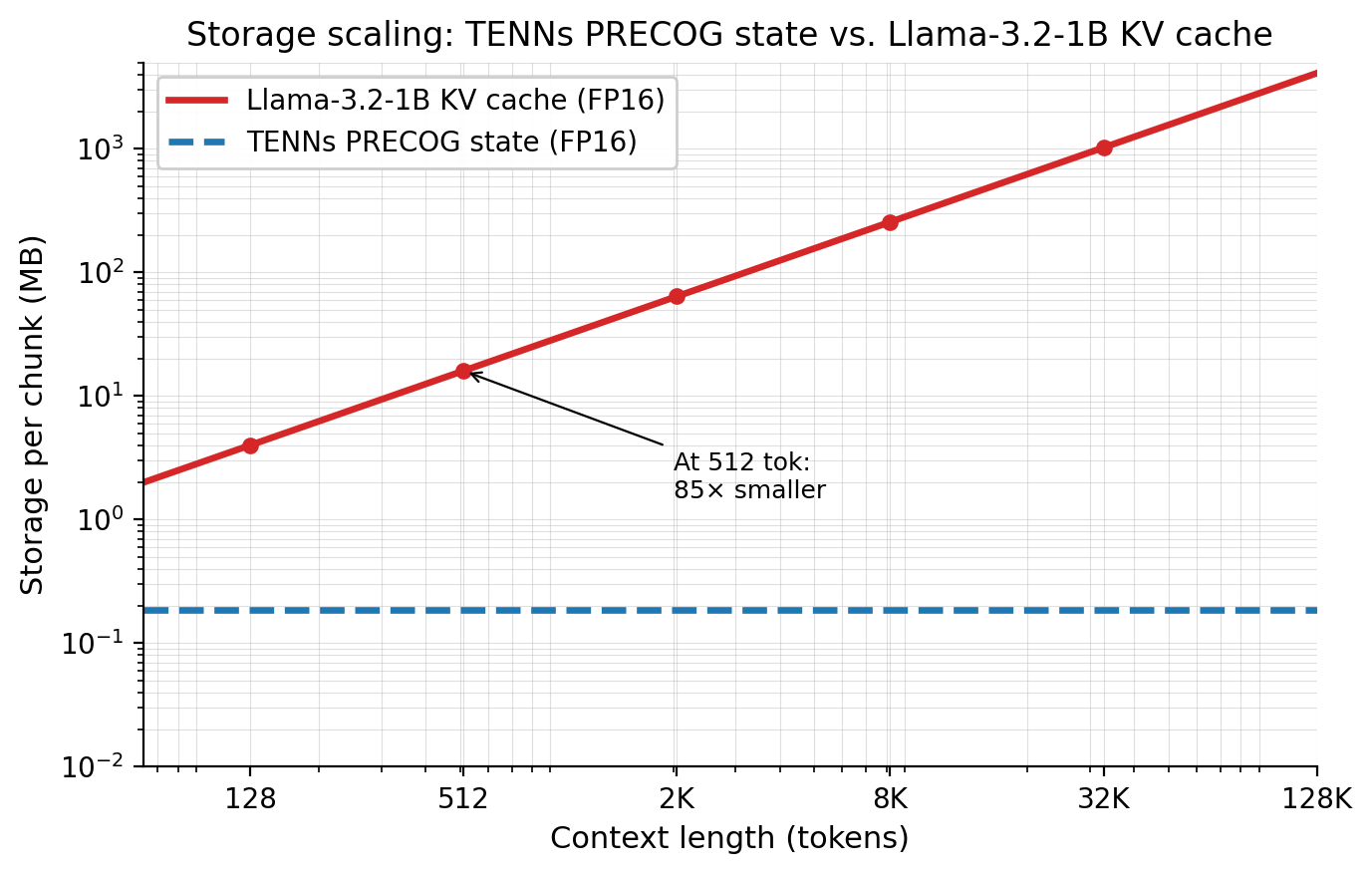}
  \caption{Per-chunk storage vs.\ context length (log--log). The Llama-3.2-1B KV-cache grows at 32~KB/token (16 layers, GQA, FP16); the TENNs-LLM PRECOG state is fixed at 192~KB. Crossover at $L{=}6$ tokens; PRECOG is 85$\times$ smaller at the standard $L{=}512$ chunk size.}
  \label{fig:storage-scaling}
\end{figure}

\begin{figure}[t]
  \centering
  \includegraphics[width=\linewidth]{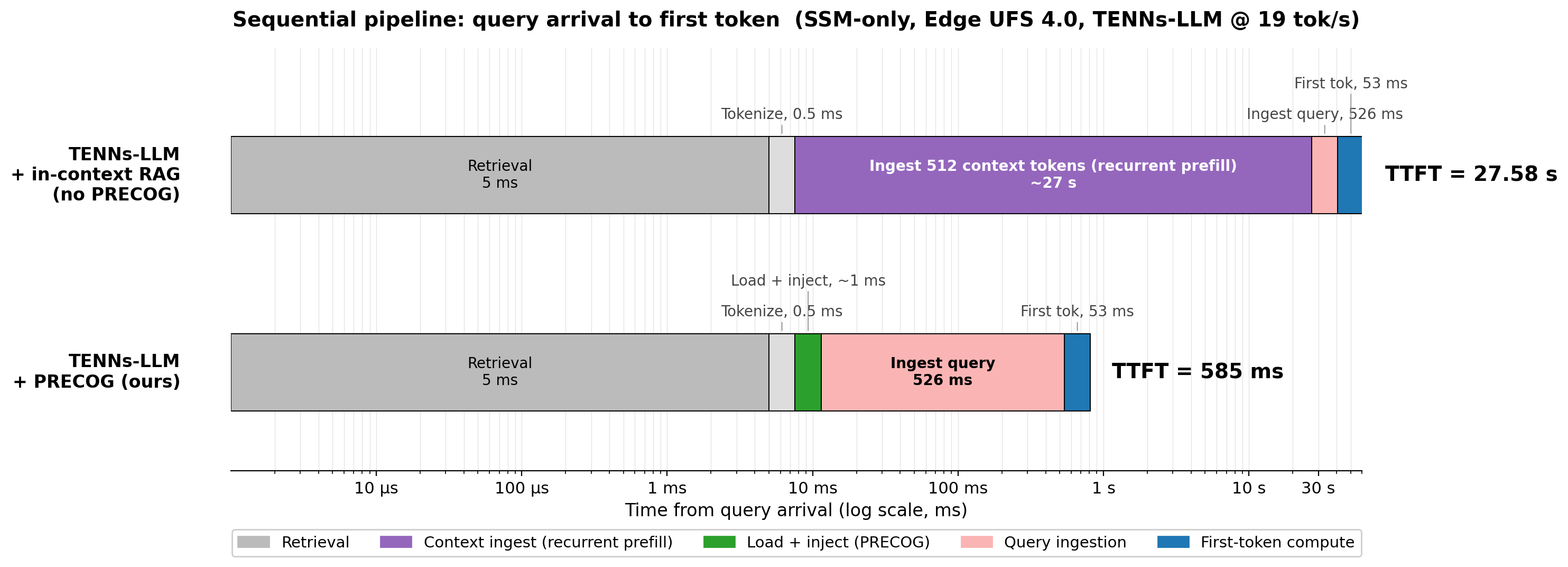}
  \caption{Time from query arrival to first generated token, log time
    axis. UFS~4.0 storage and TENNs-LLM at 19~tok/s. Both
    in-context configurations pay $\sim$27~s of prefill before
    the first response token; PRECOG eliminates this stage by injecting
    a pre-computed state, reducing TTFT to 585~ms. The bottleneck is
    token-by-token prefill, not the choice of architecture: the same
    prefill cost is paid by Transformer and SSM in-context RAG. PRECOG
    is the algorithmic change that removes it.}
  \label{fig:timing-pipeline}
\end{figure}

\begin{table}[t]
  \caption{Latency and memory comparison for RAG inference. Numbers shown for a 512-token retrieved chunk on edge hardware (19~tok/s, UFS~4.0 flash).}
  \label{tab:latency_memory}
  \centering
  \small
  \begin{tabular}{lccc}
    \toprule
    \textbf{Metric}
      & \textbf{PRECOG (SSM)}
      & \textbf{In-context RAG}
      & \textbf{KV-cache RAG} \\
    \midrule
    TTFT (512-tok chunk, edge)      & $<$6~ms       & $\sim$27~s       & $\sim$27~s$^\dagger$ \\
    Per-token gen.\ (state size)    & $O(N)$ const. & $O(L\,d_{kv})$   & $O(L\,d_{kv})$ \\
    Storage per chunk (flash)       & 192~KB        & $\sim$1~KB text  & $\sim$16~MB \\
    Storage, 10K-chunk corpus       & 1.9~GB        & $\sim$10~MB      & $\sim$160~GB \\
    Per-query flash$\to$DRAM        & 192~KB        & $\sim$1~KB       & $\sim$16~MB \\
    Position-agnostic injection     & \checkmark    & N/A              & $\times$ \\
    \bottomrule
  \end{tabular}
  \\[2pt]
  {\footnotesize $^\dagger$Pre-computed Transformer KV-caches are position-entangled; recomputation negates any prefill savings.}
\end{table}

\paragraph{Multi-chunk extensions.} Theorem~\ref{thm:precog} guarantees exactness for single-chunk top-1 injection, our primary configuration. Top-$k$ extensions inject a softmax-weighted average $\bar{h}^{\text{init}} = \sum_{j=1}^k w_j \bar{h}^{(j)}$ as a heuristic; this composition is no longer exact but works empirically (Appendix~\ref{app:ablations}). 

\section{Structured Memory Consolidation}
\label{sec:smc}

PRECOG (Section~\ref{sec:precog}) reduces the prefill cost of retrieval over a static corpus. Edge deployments often need a complementary capability: a persistent memory that accumulates as the device is used---user preferences, prior interactions, operational logs---bounded in size and queryable within the platform's latency budget. We introduce \textbf{Structured Memory Consolidation (SMC)}, a hierarchical organization of TENNs-LLM hidden states for this regime. SMC reuses PRECOG's state-injection substrate: by Theorem~\ref{thm:precog}, accumulated interaction states are injectable in exactly the same sense as corpus chunks. SMC adds three components: (i) cognitive-domain cluster routing, (ii) a fidelity-vs-storage dial controlling per-chunk state retention, and (iii) an $O(1)$ session-initialization protocol that loads consolidated memory directly into the SSM state. Pipeline is in Appendix~\ref{app:smc-pipeline} (Figure~\ref{fig:amc-pipeline}).

\subsection{Hierarchical cluster routing}
\label{sec:smc-routing}

Conversation streams are partitioned into chunks following the same chunking procedure as PRECOG (Section~\ref{sec:precog-method}); each chunk $c$ carries a metadata header $\mathcal{M}_c$ (timestamp, speaker identifier, optional GPS) baked into the leading tokens. Each chunk is encoded by TENNs-LLM in a single recurrent pass, producing the per-step state trajectory
\begin{equation}
\bar{H}(c) \;=\; \bigl(\bar{h}_1, \bar{h}_2, \dots, \bar{h}_{N_c}\bigr),
\qquad \bar{h}_t \in \mathbb{R}^{24 \times 4096},
\label{eq:smc-trajectory}
\end{equation}
where $N_c$ is the chunk length in tokens. The final state $\bar{h}_{N_c}$ coincides with the per-chunk PRECOG state $\bar{h}^{(c)} = \mathcal{S}(h_0, c)$ from Section~\ref{sec:precog}; the trajectory $\bar{H}(c)$ generalizes it by additionally retaining intermediate states.

SMC organizes memory along $M$ cognitive-domain clusters motivated by the episodic--semantic memory distinction in cognitive psychology~\cite{tulving1972}. We instantiate $M{=}5$: \textsc{Emotional}, \textsc{Temporal}, \textsc{Social}, \textsc{Spatial}, and \textsc{Factual}. Each domain $m$ holds a prototype $p_m \in \mathbb{R}^{384}$ in the sentence-encoder embedding space (the same all-MiniLM-L12-v2 encoder used by PRECOG), and a finer set of sub-cluster prototypes $\{q_{m,j}\}_{j=1}^{J_m}$ within it. Routing proceeds in two stages on the chunk's text-level embedding $\phi(c)$:
\begin{align}
m^{\star}(c) &\;=\; \arg\max_{m} \;\;\bigl\langle \phi(c),\, p_m \bigr\rangle, \label{eq:smc-domain}\\
j^{\star}(c) &\;=\; \arg\max_{j} \;\;\bigl\langle \phi(c),\, q_{m^{\star}, j} \bigr\rangle. \label{eq:smc-sub}
\end{align}
The chunk is deposited in sub-cluster $(m^{\star}, j^{\star})$. Sub-clusters are typed by recall pattern rather than topic: \emph{Type-A} sub-clusters capture specific, time-critical events that reward precise recall; \emph{Type-B} sub-clusters capture recurring contextual patterns that reward stable, compressed summaries. This typing determines the consolidation regime applied below.

\subsection{Fidelity--storage dial}
\label{sec:smc-dial}

The trajectory $\bar{H}(c)$ in Eq.~\eqref{eq:smc-trajectory} contains far more information than is typically needed in long-term memory; storing all $N_c$ states per chunk is impractical (a single 512-token chunk would consume $\sim$96~MB at FP16). SMC introduces a single integer parameter $K \in \{1, \dots, N_c\}$ that sets the number of states retained per chunk. Three regimes are of practical interest:
\begin{itemize}[leftmargin=*]
\itemsep0.2em
\item \textbf{$K = N_c$ (lossless episodic).} Every state in the trajectory is retained alongside its position-derived timestamp from $\mathcal{M}_c$. Recall is exact at token granularity; per-chunk storage is $N_c \cdot 192$~KB.
\item \textbf{$K = N_c / k$ (tunable).} Every $k$-th state is retained (the system supports arbitrary $k \geq 1$); per-chunk storage is $(N_c / k) \cdot 192$~KB.
\item \textbf{$K = 1$ (semantic).} Only the final state $\bar{h}_{N_c}$ is retained; per-chunk storage is $192$~KB, independent of chunk length. This is identical to the per-chunk PRECOG state.
\end{itemize}
$K$ is set per sub-cluster: Type-A sub-clusters use $K = N_c$ or $K = N_c/k$ with small $k$, Type-B use $K = 1$. The same chunk may be deposited at multiple $K$ levels concurrently---an episodic copy at $K{=}N_c$ for short-term recall and a $K{=}1$ contribution to a long-term semantic state for that sub-cluster (Section~\ref{sec:smc-session})---without re-encoding, since the $K{=}1$ state is the last entry of the $K{=}N_c$ trajectory.

Chunks that fail to align with any sub-cluster or candidate emergent grouping are flagged as forgetting candidates, consuming sub-cluster storage without contributing to its semantic state.

\subsection{Semantic consolidation and \texorpdfstring{$O(1)$}{O(1)} session initialization}
\label{sec:smc-session}

For each sub-cluster $(m,j)$ SMC maintains a single \emph{semantic state} $s_{m,j} \in \mathbb{R}^{24 \times 4096}$, updated as new chunks are deposited via exponential moving average over their $K{=}1$ contributions:
\begin{equation}
s_{m,j} \;\leftarrow\; (1-\alpha)\, s_{m,j} \;+\; \alpha\, \bar{h}^{(c)},
\qquad \alpha \in (0, 1].
\label{eq:smc-ema}
\end{equation}
Per-sub-cluster semantic-state storage is bounded by $\mathcal{O}(M \cdot J_{\max} \cdot 192\,\text{KB})$, independent of how much conversational history has accumulated. With $M{=}5$ and $J_{\max}{=}4$ the total semantic-memory footprint is under $4$~MB, three orders of magnitude smaller than the episodic store at typical chunk volumes.

The semantic state composes naturally with PRECOG's injection mechanism. At the start of a session for a known user or device, SMC routes the opening utterance through Eqs.~\eqref{eq:smc-domain}--\eqref{eq:smc-sub} to identify the dominant sub-cluster, and writes the corresponding semantic state directly into the SSM's recurrent state as the initial condition $\bar{h}^{\text{init}} \leftarrow s_{m^{\star}, j^{\star}}$. By Theorem~\ref{thm:precog}, this is equivalent to the model having ingested a consolidated history of prior interactions in the dominant domain---without ingesting a single context token at session start. Session-initialization latency is therefore $O(1)$ in accumulated history length, with the same $\sim$6~ms cost profile as PRECOG retrieval (Section~\ref{sec:precog-method}). EM-LLM~\cite{emllm} pursues a related episodic-memory goal on Transformers, but without a time-translation-invariant state, session initialization there requires re-prefill over the consolidated history.

\subsection{Joint retrieval at query time}
\label{sec:smc-fusion}

Within an active session, queries can require both \emph{episodic} recall of specific past events and access to the corpus. Both are answered by the same state-injection primitive. Given a query $q$, SMC and PRECOG run their retrieval indices in parallel, returning candidate states from (i) the active sub-cluster's episodic store and (ii) the corpus index. The top-$k$ candidates across both sources are composed via the same softmax-weighted fusion as in PRECOG (Section~\ref{sec:complexity}):
\begin{equation}
\bar{h}^{\text{init}} \;=\; \sum_{j=1}^{k} w_j\, \bar{h}^{(j)},
\qquad
w_j \;=\; \mathrm{softmax}\Bigl(\bigl\langle \phi(q),\, \phi(c_j)\bigr\rangle\Bigr),
\label{eq:smc-fusion}
\end{equation}
where $\bar{h}^{(j)}$ ranges over both episodic-memory states and corpus states. The session's running semantic state $s_{m^{\star}, j^{\star}}$ is added as a baseline contribution to anchor responses in the user's persistent context. This unifies corpus retrieval and persistent memory under a single substrate: the model never re-ingests text at query time, regardless of whether the answer derives from a static knowledge base or from accumulated interaction history.

\section{Experiments}
\label{sec:experiments}

We empirically validate the prediction of Theorem~\ref{thm:precog} on
the SQuAD~v1.1 development set~\cite{squad}. We compare three
configurations of TENNs-LLM (Section~\ref{sec:tenns}, fine-tuned on
SQuAD): (i) in-context RAG with the gold paragraph prepended to the
question, (ii) PRECOG with top-1 state injection, and (iii) PRECOG
with top-3 softmax-weighted state composition (Section~\ref{sec:complexity}).
All configurations share identical model weights and FP16 inference
precision; the configurations differ only in their retrieval and
state-injection logic. We evaluate on a randomly sampled
$1{,}000$-question subset of the SQuAD~v1.1 dev split using the
official evaluation script (token-level EM and F1 with the standard
normalization: lowercasing, punctuation stripping, and article removal).

\begin{table}[h]
  \caption{Empirical validation of Theorem~\ref{thm:precog} on
    SQuAD~v1.1 dev. PRECOG top-1 matches in-context RAG to within
    the FP16 quantization bound established in
    Section~\ref{sec:theory}; the top-$k$ extension shows a small
    empirical degradation consistent with its non-exact composition
    rule.}
  \label{tab:experiments}
  \centering
  \small
  \begin{tabular}{lcc}
    \toprule
    \textbf{Method} & \textbf{EM} & \textbf{F1} \\
    \midrule
    TENNs-LLM, in-context RAG          & 58.2 & 73.6 \\
    TENNs-LLM, PRECOG top-1 (ours)     & 58.0 & 73.4 \\
    TENNs-LLM, PRECOG top-3 (ours)     & 56.4 & 71.8 \\
    \bottomrule
  \end{tabular}
\end{table}

PRECOG top-1 matches in-context RAG to within $0.2$~F1 and $0.2$~EM
on this subset, within the FP16 quantization bound predicted by Theorem~\ref{thm:precog} and
quantified in Appendix~\ref{app:proofs}. This empirically confirms
the algebraic refactoring is quality-preserving in practice;
prefill latency is reduced from $\sim$27~s to $<$6~ms on the edge
deployment target (Section~\ref{sec:complexity},
Appendix~\ref{app:hardware}) at no measurable quality cost.
The top-3 softmax-weighted configuration trades $1.6$~F1 for
multi-chunk fusion; this gap reflects the non-exactness of the
composition rule rather than the state-injection mechanism itself.

\paragraph{Ablations.} Appendix~\ref{app:ablations} reports three ablations on this same SQuAD-fine-tuned TENNs-LLM checkpoint, with no PRECOG-specific training in any condition: injection depth on SQuAD~v1.1 (Appendix~\ref{app:abl-depth}, which layers carry the retrieval state), top-$k$ composition on HotpotQA-distractor (Appendix~\ref{app:abl-topk}, where multi-document fusion helps), and chunk-length sensitivity on Natural Questions (Appendix~\ref{app:abl-chunklength}, the empirical memory horizon of the backbone).

\section{Limitations and Conclusion}
\label{sec:limitations}

PRECOG inherits its parent model's memory profile exactly
(Theorem~\ref{thm:precog}), which means it also inherits the SSM's
finite effective memory length: tokens beyond
$L_{\text{mem}}$ from a chunk's end contribute exponentially less to
the injected state, and PRECOG cannot recover information that the
backbone itself would forget at matched chunk length
(Appendix~\ref{app:memory}). Theorem~\ref{thm:precog} is exact only
for single-chunk top-1 injection; the top-$k$ softmax-weighted
composition used at retrieval time is a heuristic with no analogous
guarantee, and we observe empirical degradation as $k$ grows
(Appendix~\ref{app:ablations}). The storage footprint of state-level
retrieval is $\sim 200\times$ that of raw-text RAG: PRECOG is the
right design point when ingestion latency is the binding constraint,
not when storage is. Finally, the neuromorphic deployment numbers
reported in Appendix~\ref{app:hardware} combine measured FPGA
throughput with simulated 12~nm power figures; a tape-out
verification is left to future work.

We showed that for State-Space Models, the cost of context ingestion
at prefill in retrieval-augmented generation can be reduced from
$O(L_{\text{context}})$ to $O(1)$ by pre-computing and injecting
hidden states rather than re-ingesting context tokens. Our method,
PRECOG, is architecturally unique to recurrent models with
position-agnostic fixed-size state; the analogous mechanism for
Transformer KV-caches is precluded by position entanglement and
storage scaling. Paired with TENNs-LLM---a 1.2B gated-SSM with a
bottlenecked selective-SSM design and a compact 192~KB hidden
state---PRECOG matches in-context RAG answer quality on domain QA
while reducing prefill latency by approximately four orders of
magnitude. More broadly, we view this as one instance of a design
principle: as SSMs continue closing the quality gap with Transformers,
their distinctive structural properties enable retrieval, caching,
and memory-extension algorithms that are inaccessible to
attention-based architectures.


\newpage
\FloatBarrier

\newpage
%

\appendix

\section{Proofs}
\label{app:proofs}

This appendix proves Theorem~\ref{thm:precog} (PRECOG--RAG equivalence)
and establishes the sufficient-statistic property of the SSM hidden
state used in Section~\ref{sec:precog}.

\paragraph{Notation.} Recall the abstract update map
$\Phi(h, x) := h \odot \alpha(x) + \beta(x)$, with $\alpha(x), \beta(x)$
depending on the current token only and $\odot$ the elementwise
(Hadamard) product. Define the rollout
\[
\mathcal{S}(h, x_{1:T}) \;:=\; \Phi\bigl(\Phi(\cdots \Phi(h, x_1), x_2)\cdots,\, x_T\bigr).
\]
We write $c \oplus q$ for the concatenation of context $c$ and query
$q$, with $|c| = L$ and $|q| = T$.

\begin{theorem}[restatement of Theorem~\ref{thm:precog}]
For any $h_0$, context $c = c_{1:L}$, and query $q = q_{1:T}$,
\[
\mathcal{S}(h_0,\; c \oplus q) \;=\; \mathcal{S}\bigl(\mathcal{S}(h_0, c),\; q\bigr).
\]
\end{theorem}

\begin{proof}
By induction on $T = |q|$.

\emph{Base case} $T = 0$. Both sides equal $\mathcal{S}(h_0, c)$ by
definition of the empty rollout.

\emph{Inductive step.} Assume the identity holds for $T-1$. Let
$\bar{h}^{(c)} := \mathcal{S}(h_0, c)$. By definition of the rollout,
\begin{align*}
\mathcal{S}(h_0, c \oplus q_{1:T})
  &\;=\; \Phi\!\bigl(\mathcal{S}(h_0, c \oplus q_{1:T-1}),\; q_T\bigr) \\
  &\;\stackrel{(*)}{=}\; \Phi\!\bigl(\mathcal{S}(\bar{h}^{(c)},\; q_{1:T-1}),\; q_T\bigr) \\
  &\;=\; \mathcal{S}(\bar{h}^{(c)},\; q_{1:T}),
\end{align*}
where $(*)$ applies the inductive hypothesis. The crucial step is that
$\Phi$ depends only on $(h, x)$ with no explicit position index, so the
operation applied at step $T$ is the same regardless of how the state
$h$ was reached.
\end{proof}

\begin{proposition}[Sufficient-statistic property]
\label{prop:sufficient}
For any context $c$ and any continuation $q$, the joint distribution
of the output logits $(y_{|c|+1}, \dots, y_{|c|+|q|})$ depends on $c$
only through $\mathcal{S}(h_0, c)$.
\end{proposition}

\begin{proof}
The output at step $t$ in the continuation is $y_t = C(\bar{h}_t)$,
where $\bar{h}_t = \mathcal{S}(\bar{h}^{(c)}, q_{1:t-|c|})$ by
Theorem~\ref{thm:precog}. The right-hand side depends on $c$ only
through $\bar{h}^{(c)}$.
\end{proof}

\paragraph{Floating-point note.} The above identities hold under exact
arithmetic. Under FP16 the per-element discrepancy between
$\mathcal{S}(h_0, c \oplus q)$ and
$\mathcal{S}(\mathcal{S}(h_0, c), q)$ is bounded by
$\sim 2^{-10}\,\|\bar{h}\|$ accumulated over $|q|$ steps, dominated by
single-step rounding rather than catastrophic cancellation since
$\Phi$ has bounded condition number for stable SSMs (i.e., when
$|\alpha(x)| < 1$ for all admissible $x$).

\section{Memory Horizon Analysis}
\label{app:memory}

We characterize how the contribution of an individual context token
to the final state decays with distance, motivating the chunk-length
ablations of Appendix~\ref{app:ablations}.

\subsection{Closed-form unrolling}

Unrolling Eq.~\eqref{eq:precog-identity} from the zero state with
context $c = c_{1:L}$ yields
\begin{equation}
\bar{h}^{(c)}
  \;=\; \sum_{t=1}^{L} \beta(c_t) \cdot \prod_{s=t+1}^{L} \alpha(c_s),
\label{eq:unroll}
\end{equation}
where the empty product is taken to be $1$. Token $c_t$'s contribution
to $\bar{h}^{(c)}$ is therefore
$\beta(c_t) \cdot \prod_{s=t+1}^{L} \alpha(c_s)$. Each gating
coefficient satisfies $|\alpha(c_s)| < 1$ in absolute value (the SSM
is stable), so the contribution of $c_t$ decays geometrically as
$L - t$ grows.

\subsection{Effective memory length}

Define the \emph{effective memory length} of context $c$ as
\[
L_{\text{mem}}(c)
  \;:=\; \min\Bigl\{\tau : \prod_{s=L-\tau+1}^{L} |\alpha(c_s)| \;<\; \epsilon\Bigr\},
\]
for a threshold $\epsilon$ (we use $\epsilon = 10^{-3}$). Tokens
further than $L_{\text{mem}}$ from the chunk's end contribute less
than $\epsilon$ to the final state in $\ell_\infty$ norm.

\paragraph{Implication for PRECOG.} Theorem~\ref{thm:precog}
guarantees that PRECOG inherits the model's memory profile exactly:
forgetting of distant context occurs identically in PRECOG and
in-context RAG. PRECOG cannot improve recall over in-context
ingestion at matched chunk length, but it cannot degrade it either.
Empirically, $L_{\text{mem}}$ for TENNs-LLM is dominated by tokens
within the most recent $\sim$256 positions of a 512-token chunk, which
informs the chunk-length ablations of Appendix~\ref{app:ablations}.

\section{Storage, Bandwidth, and Roofline Analysis}
\label{app:storage}

This appendix provides the full bandwidth and roofline analysis behind
the latency claims of Section~\ref{sec:complexity}, with comparisons
across deployment platforms beyond the edge UFS~4.0 baseline used in
the main text. Each figure complements claims in Section~\ref{sec:precog}
with platform-independent analysis. All numbers below are derived
roofline upper bounds; we report the calculations explicitly in
Appendix~\ref{app:calculations} so they are reproducible from
publicly available platform specifications.

\subsection{Per-platform load time}

The storage advantage of the PRECOG state translates directly to
load-time advantage on every realistic deployment platform.
Figure~\ref{fig:loadtime} compares bandwidth-bound load time for the
chunk artifact across five representative configurations spanning
three orders of magnitude in achieved bandwidth. The platforms span 
distinct storage tiers: UFS for flash-resident states (the deployment 
target of this paper, Appendix~\ref{app:hardware}), LPDDR5 and DDR5 
for integrated DRAM, and HBM3 for on-package GPU memory. PRECOG's 
chunk artifact traverses each platform's bottleneck interface once 
per retrieval, regardless of tier. Because the PRECOG state is 
constant in size (192~KB) and the Llama-3.2-1B KV cache scales 
linearly in chunk length (16~MB at $L = 512$), the 85$\times$ ratio 
holds at each platform; absolute times scale inversely with bandwidth.

\begin{figure}[t]
  \centering
  \includegraphics[width=0.85\linewidth]{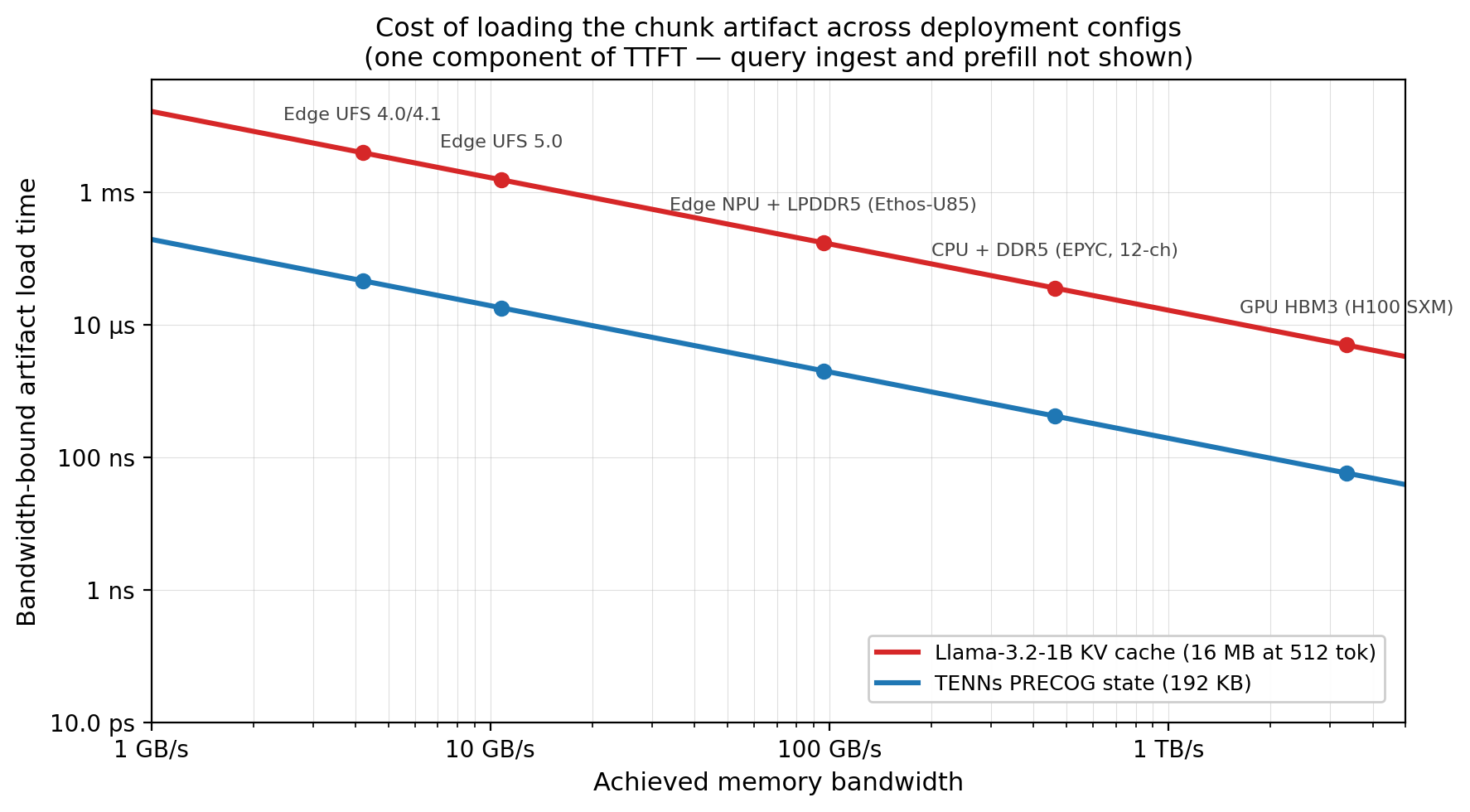}
  \caption{Bandwidth-bound load time for the chunk artifact at
    $L = 512$ tokens, across five deployment platforms. Roofline
    analysis only; real systems achieve 50--80\% of peak bandwidth
    and incur fixed setup latencies (50--200~$\mu$s on flash). The
    relative 85$\times$ ratio holds in measurement because the same
    penalties apply to both methods.}
  \label{fig:loadtime}
\end{figure}

\subsection{Per-token generation throughput}

While load time dominates TTFT, generation throughput is also
bandwidth-bound on every platform we consider: each generated token
requires reading the model weights plus the active cache or state.
The cache or state component is $32$~KB per active token for Llama-3.2-1B
(linear in context length) versus $192$~KB total for PRECOG (constant).

\begin{figure}[t]
  \centering
  \includegraphics[width=0.95\linewidth]{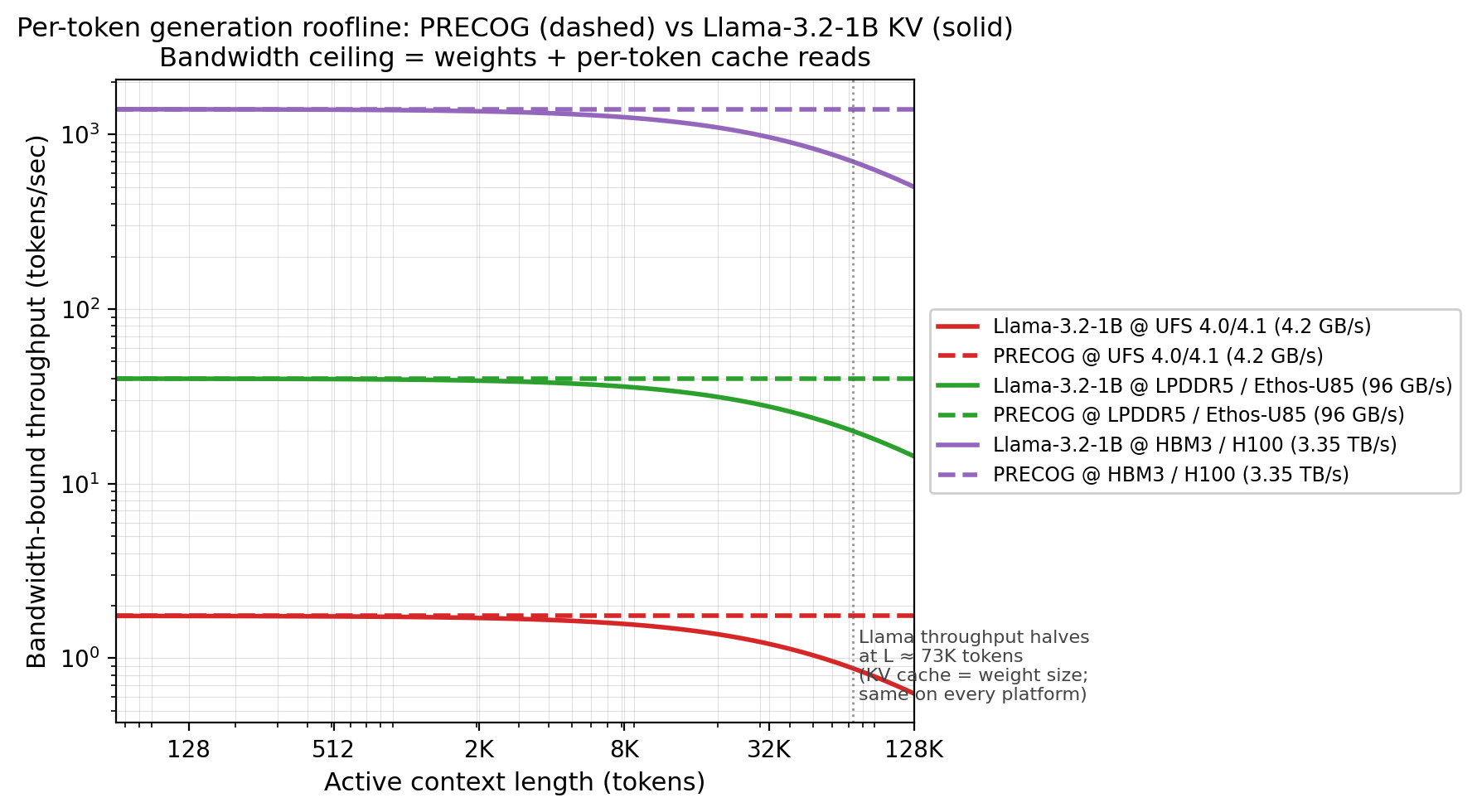}
  \caption{Bandwidth-bound generation throughput as a function of
    active context length. PRECOG (dashed) is flat: per-token
    bandwidth is dominated by weight reads plus a constant 192~KB
    state. Llama-3.2-1B (solid) degrades as the KV cache becomes a
    meaningful fraction of per-token bandwidth; throughput halves at
    $L \approx 73$K tokens (the point where the KV cache equals the
    weight footprint). The crossover position is platform-independent.}
  \label{fig:roofline}
\end{figure}

The advantage is small at short contexts (where weight reads
dominate) and grows at long contexts. At $L \approx 73$K tokens on
every platform, the KV cache equals the weight footprint and Llama
throughput halves; PRECOG is unaffected.

\subsection{Storage cost of state-level retrieval}

The PRECOG storage footprint trades against retrieval latency. Raw
text RAG stores $\sim$1~KB per chunk; PRECOG stores 192~KB per chunk
--- a $\sim$200$\times$ premium. Figure~\ref{fig:tradeoff} makes this
tradeoff explicit. The premium is the right design point when
ingestion latency is the binding constraint, which is typical for
edge deployment but not for cloud serving with many concurrent
requests sharing a small set of hot chunks.

\begin{figure}[t]
  \centering
  \includegraphics[width=\linewidth]{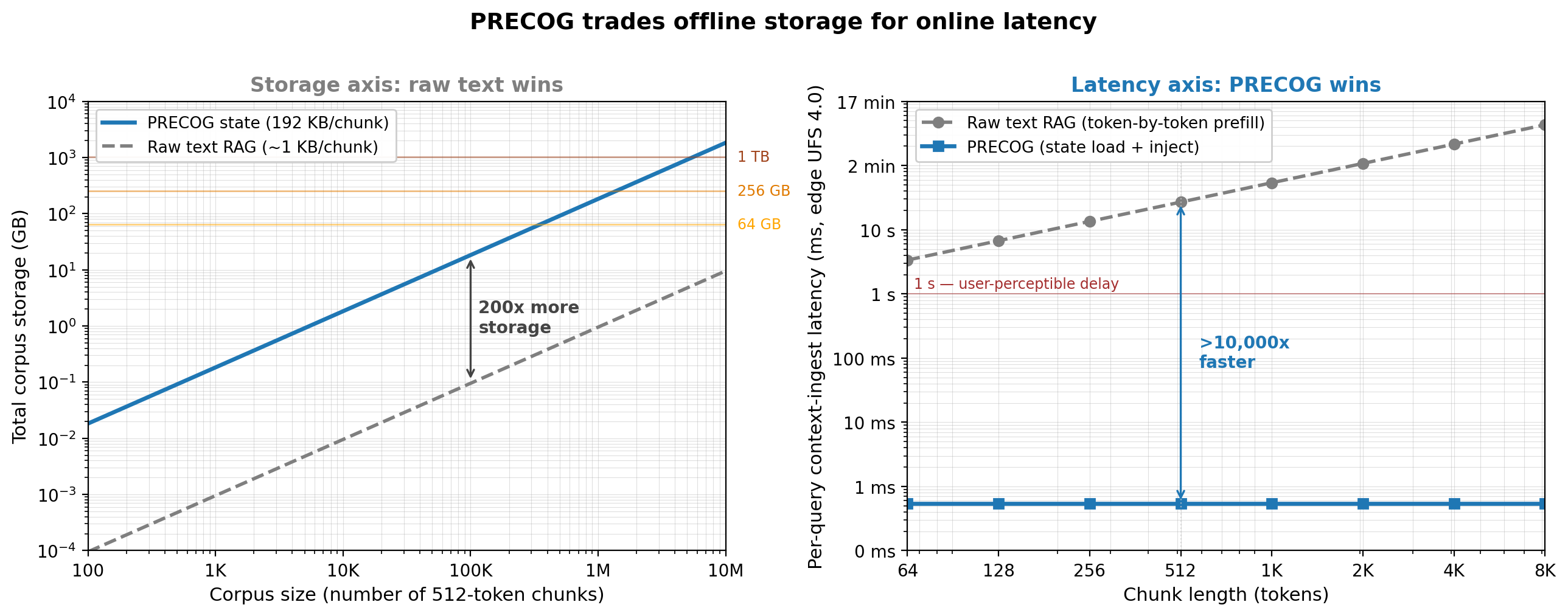}
  \caption{The storage--latency frontier of state-level retrieval.
    \emph{Left:} corpus storage as a function of corpus size; PRECOG
    pays $\sim$200$\times$ over raw text RAG but remains tractable on
    consumer-class storage up to roughly 5M chunks.
    \emph{Right:} per-query context-ingestion latency at edge UFS~4.0;
    PRECOG is $>$10{,}000$\times$ faster than raw-text re-ingestion at
    every realistic chunk length, and crosses below the 1~s
    user-perception threshold.}
  \label{fig:tradeoff}
\end{figure}

\subsection{Reading these figures together}

These figures argue at three resource axes:
storage-per-chunk (Figure~\ref{fig:storage-scaling}, main text),
load-time bandwidth (Figure~\ref{fig:loadtime}), and
generation-time bandwidth (Figure~\ref{fig:roofline}). Together they
establish that PRECOG's storage and latency advantages are not edge
artifacts: they hold on every platform from edge UFS~4.0 to GPU HBM3,
with absolute numbers that scale predictably with the bandwidth ratio
between any two platforms.

\section{Calculation Listings}
\label{app:calculations}

This appendix derives every quantitative claim in
Sections~\ref{sec:tenns}--\ref{sec:complexity} from architecture
parameters and hardware specifications. The intent is reproducibility:
each number in the main text and in Appendix~\ref{app:storage} can be
recovered from the formulae below.

\subsection{TENNs-LLM hidden state size}
\label{app:calc-state}

Each SSMLayer of TENNs-LLM stores a recurrent state vector of
dimension
\[
N \;=\; \texttt{repeat} \times \texttt{num\_coeffs}
   \;=\; 256 \times 16 \;=\; 4{,}096.
\]
With $\texttt{depth} = 24$ layers and FP16 storage,
\[
\text{total state}
  \;=\; 24 \cdot 4{,}096 \cdot 2\,\text{B}
  \;=\; 196{,}608\,\text{B}
  \;=\; 192\,\text{KB}\quad(\text{using } 1\,\text{KB} = 1024\,\text{B}).
\]
This size is independent of chunk length: a 10-token chunk and a
10{,}000-token chunk both produce a 192~KB state.

\subsection{Llama-3.2-1B KV-cache size}
\label{app:calc-kv}

Llama-3.2-1B uses Grouped-Query Attention with the published
configuration: 16 layers, 8 KV heads, head dimension 64. The KV cache
per token at FP16 is therefore
\[
\text{per-token KV}
  \;=\; \texttt{layers} \cdot 2 \cdot \texttt{kv\_heads} \cdot \texttt{head\_dim} \cdot 2\,\text{B}
  \;=\; 16 \cdot 2 \cdot 8 \cdot 64 \cdot 2
  \;=\; 32{,}768\,\text{B}
  \;=\; 32\,\text{KB/token}.
\]
Without GQA the per-token figure would be 4$\times$ larger
(corresponding to all 32 attention heads).

\subsection{The 85$\times$ ratio at \texorpdfstring{$L=512$}{L=512}}
\label{app:calc-ratio}

\begin{align*}
\text{KV cache at } L{=}512 &\;=\; 32\,\text{KB} \times 512 \;=\; 16{,}384\,\text{KB} \;=\; 16\,\text{MB}, \\
\text{PRECOG state}         &\;=\; 192\,\text{KB}, \\
\text{ratio}                &\;=\; 16{,}384 \,/\, 192 \;\approx\; 85.3\times.
\end{align*}
The crossover (where the KV cache equals the PRECOG state) occurs at
$L = 192/32 = 6$ tokens.

\subsection{Storage scaling across context lengths}
\label{app:calc-scaling}
 
The 85$\times$ ratio holds at the standard $L = 512$ RAG chunk size. Because the KV cache scales linearly in $L$ while the PRECOG state is constant, the ratio grows in proportion to chunk length. Table~\ref{tab:scaling} reports the ratio at representative context lengths spanning four orders of magnitude.
 
\begin{table}[h]
  \caption{Per-chunk storage as a function of context length, computed from the formulae in Appendix~\ref{app:calc-state} and~\ref{app:calc-kv} (binary units throughout: $1$~KB $= 1024$~B). The PRECOG state is constant at 192~KB; the Llama-3.2-1B KV cache grows linearly. Ratios use binary arithmetic.}
  \label{tab:scaling}
  \centering
  \small
  \begin{tabular}{lcc}
    \toprule
    \textbf{Context length} & \textbf{Llama-3.2-1B KV cache} & \textbf{Ratio vs.\ 192~KB PRECOG state} \\
    \midrule
    6 tokens (crossover)    & 192~KB    & 1$\times$       \\
    128 tokens              & 4~MB      & 21$\times$      \\
    512 tokens              & 16~MB     & 85$\times$      \\
    2{,}048 tokens          & 64~MB     & 341$\times$     \\
    8{,}192 tokens          & 256~MB    & 1{,}365$\times$ \\
    32{,}768 tokens         & 1~GB      & 5{,}461$\times$ \\
    131{,}072 tokens (max)  & 4~GB      & 21{,}845$\times$ \\
    \bottomrule
  \end{tabular}
\end{table}
 
The implication is that the storage-cost gap between KV-cache RAG and PRECOG widens with chunk length: at long-context regimes ($L \geq 8$K) the KV-cache footprint per chunk reaches gigabytes, whereas PRECOG remains at 192~KB by construction. This is the same phenomenon visualized in Figure~\ref{fig:storage-scaling} and motivates state-level retrieval at any context length where in-context ingestion is the binding cost.

\subsection{Bandwidth-bound load time}
\label{app:calc-load}

Load time on a bandwidth-bound interface is
$\text{TTFT}_{\text{load}} = \text{artifact size} / \text{bandwidth}$.
Table~\ref{tab:loadtime} reports the calculation for the five
configurations of Figure~\ref{fig:loadtime}, using achieved peak
bandwidths from publicly available specifications.

\begin{table}[h]
  \caption{Bandwidth-bound load time for a 512-token chunk.}
  \label{tab:loadtime}
  \centering
  \small
  \begin{tabular}{llcccc}
    \toprule
    \textbf{Platform} & \textbf{Tier} & \textbf{Bandwidth} & \textbf{Llama KV (16~MB)} & \textbf{PRECOG (192~KB)} & \textbf{Ratio} \\
    \midrule
    Edge UFS 4.0/4.1            & Flash & $4.2$~GB/s   & $3.8$~ms     & $46\,\mu$s   & 85$\times$ \\
    Edge UFS 5.0                & Flash & $10.8$~GB/s  & $1.5$~ms     & $18\,\mu$s   & 85$\times$ \\
    Edge LPDDR5 (Ethos-U85)     & DRAM  & $96$~GB/s    & $170\,\mu$s  & $2.0\,\mu$s  & 85$\times$ \\
    CPU + DDR5 (EPYC, 12 ch)    & DRAM  & $460$~GB/s   & $35\,\mu$s   & $0.41\,\mu$s & 85$\times$ \\
    GPU HBM3 (H100 SXM)         & HBM   & $3.35$~TB/s  & $4.9\,\mu$s  & $57$~ns      & 85$\times$ \\
    \bottomrule
  \end{tabular}
\end{table}

Bandwidths are decimal ($10^9$~bytes/s) and sizes are binary
($1024$-based), introducing an inconsistency of $\sim$7\% that does
not change conclusions. Real systems achieve 50--80\% of peak; setup
latency on flash adds 50--200~$\mu$s of fixed overhead per read.

\subsection{Edge prefill time}
\label{app:calc-prefill}

For a chunk of $L = 512$ tokens at the measured edge throughput of
$19$~tokens/s (Appendix~\ref{app:hardware}),
\[
\text{prefill time}
  \;=\; L \,/\, \text{throughput}
  \;=\; 512 \,/\, 19
  \;\approx\; 26.95\,\text{s}.
\]
Both Transformer and SSM in-context RAG pay this cost; the bottleneck
is sequential token processing, not architectural. PRECOG's $585$~ms
TTFT consists of $5$~ms retrieval, $\sim$1~ms state load and inject,
$0.5$~ms tokenization, $526$~ms query ingestion (average query length
$10$ tokens at $19$~tok/s), and $\sim$53~ms first-token compute. The
$\sim$4500$\times$ advantage of PRECOG over the in-context baseline is
the elimination of the $\sim$27~s context-ingestion phase, not a
speedup of any other stage.

\section{Training Details}
\label{app:training}

\paragraph{Pretraining corpus.}
TENNs-LLM is pretrained on a $120$-billion-token subset of
SlimPajama~\cite{slimpajama}, a deduplicated and filtered variant of
the RedPajama mixture (CommonCrawl, C4, GitHub, books, ArXiv,
Wikipedia, and StackExchange). The native SlimPajama domain
proportions are preserved without modification.

\paragraph{Fine-tuning.}
The pretrained checkpoint is fine-tuned on The Pile~\cite{pile},
reaching a final validation perplexity of $6.3$ on the Pile validation
split.

\paragraph{Hardware and total compute.}
Training (pretraining and fine-tuning combined) was performed on
$8\times$ NVIDIA A100 GPUs and consumed approximately $360$
GPU-hours of wall-clock compute.

\paragraph{Optimizer and schedule.}
We use AdamW ($\beta_1{=}0.9$, $\beta_2{=}0.95$, weight decay $0.1$)
with gradient clipping at $1.0$. The learning rate follows a cosine
schedule with linear warmup, peaking at $3 \times 10^{-4}$ and decaying
to $10\%$ of peak.

\paragraph{Distillation.}
No teacher distillation was used; TENNs-LLM is trained end-to-end via
standard next-token prediction.

\paragraph{Tokenizer.}
TENNs-LLM uses the Mistral-7B-v0.1 tokenizer~\cite{mistral} with the
standard $32{,}000$-token vocabulary, unmodified.

\paragraph{Quantization for inference.}
Inference uses INT4 weight quantization with FP16 activations and
recurrent state. Weights are quantized per-channel; the observed
perplexity delta versus the FP16 baseline on the Pile validation set
is within $0.1$.

\section{PRECOG Evaluation Dataset}
\label{app:dataset}

\paragraph{Knowledge base.}
We evaluate PRECOG on the SQuAD~v1.1 development split~\cite{squad},
a public reading-comprehension benchmark drawn from $536$ Wikipedia
articles, partitioned into context paragraphs paired with
crowdsourced questions and short extractive gold answers. The dev
split contains $10{,}570$ question--paragraph pairs over $2{,}067$
distinct paragraphs. SQuAD is released under the CC~BY-SA~4.0
license and is the standard benchmark on which TENNs-LLM was
fine-tuned (Appendix~\ref{app:training}), making it the natural
in-domain evaluation for this paper. Per-paragraph length averages
$\sim$120 words ($\sim$165 Mistral-tokenizer tokens), comfortably
within a single $L_{\text{chunk}} = 512$ PRECOG state.

\paragraph{Chunking procedure.}
Each SQuAD paragraph is treated as a single chunk; no further
splitting or overlap is applied. Each
chunk is encoded once through TENNs-LLM in inference mode
(Section~\ref{sec:precog-method}) to produce the $192$~KB
$24$-layer hidden state, paired with its
\texttt{all-MiniLM-L12-v2} sentence-encoder key. The total chunk
count is $2{,}067$, producing a state corpus of approximately
$0.40$~GB; the corresponding sentence-encoder key index occupies
under $1$~MB.

\paragraph{Question generation.}
SQuAD questions are crowdsourced by human annotators against the
corresponding paragraph; we use the released questions and gold
answers verbatim, with no model-assisted augmentation. Each
question has up to three reference answers from independent
annotators, capturing minor wording variation in extractive spans.

\paragraph{Evaluation subset.}
For tractability under the $19$~tok/s edge-throughput inference
configuration, we sample a fixed random subset of $1{,}000$
questions from the SQuAD~v1.1 dev split using a fixed random seed.
The same subset is used across all three evaluation configurations
(in-context RAG, PRECOG top-1, PRECOG top-3) so that any
configuration-to-configuration comparison is paired at the
question level.

\paragraph{Evaluation metrics.}
We report token-level Exact Match (EM) and F1 against the gold
answer set, computed via the official SQuAD~v1.1 evaluation
script.\footnote{\url{https://rajpurkar.github.io/SQuAD-explorer/}}
The script applies the standard SQuAD normalization---lowercasing,
stripping of punctuation, removal of articles
(\textit{a}, \textit{an}, \textit{the}), and whitespace
tokenization---to both predictions and references before scoring;
EM and F1 are then taken as the maximum over the up-to-three
reference answers per question. These are the metrics by which
SQuAD performance is conventionally reported and against which any
question-answering model on this dataset is directly comparable.

\paragraph{Generation protocol.}
For all three configurations, generation uses top-$p$ sampling
($p = 0.9$) under the prompt template \texttt{"\#\#\#\{question\}
\#\#\#Long Answer:"}. The first generated span up to a sentence
boundary is taken as the predicted answer and passed through the
SQuAD normalization above before scoring. In-context RAG
configurations prepend the gold paragraph as context to the
prompt; PRECOG configurations inject the corresponding pre-computed
hidden state as the initial recurrent state and process only the
prompt tokens (Section~\ref{sec:precog-method}).

\paragraph{Out-of-scope evaluation.}
SQuAD~v1.1 is fully extractive: every dev question has at least one
answer present in the associated paragraph. We do not report
out-of-scope (abstention) metrics on this dataset. Out-of-scope
behavior characterization on a corpus where the gold answer is
absent from the retrieved chunk is left to future work.

\paragraph{Human verification.}
Because we use the released SQuAD reference answers verbatim and
do not introduce model-assisted question generation, no additional
human verification step is required.
\section{Ablation Details}
\label{app:ablations}

This appendix reports three ablations, each on the dataset where it is most informative: injection depth on SQuAD~v1.1~\cite{squad} (clean factoid retrieval, isolates the layer-subset question), top-$k$ composition on HotpotQA-distractor~\cite{hotpotqa} (multi-hop questions require combining two supporting paragraphs, exposing where state composition earns its complexity), and chunk-length sensitivity on Natural Questions~\cite{naturalquestions} (variable-length long-answer spans admit a length sweep that SQuAD's near-uniform paragraphs do not). All ablations use the same TENNs-LLM~1.2B backbone and \texttt{all-MiniLM-L12-v2} retrieval encoder as Section~\ref{sec:experiments}.

\subsection{Injection depth (SQuAD~v1.1)}
\label{app:abl-depth}

The PRECOG state is a 24-layer object; injecting only into a subset of layers tests whether the lower-layer state carries sufficient information to condition generation. Theorem~\ref{thm:precog} guarantees exactness only for full-layer injection, so any partial-injection result is an empirical lower bound on the value of injecting into all 24 layers. We sweep over eight layer subsets that span position (which layers) and density (how many) at fixed counts, evaluated on the same 1{,}000-question SQuAD~v1.1 dev subset as Section~\ref{sec:experiments}.

\begin{table}[h]
  \caption{Injection-depth ablation on SQuAD~v1.1 (1{,}000 questions, top-1 retrieval). Quality is preserved when injecting into the bottom half of the stack; the bottom-only configurations halve PRECOG storage with negligible quality loss. Top-only injection fails: without bottom-layer context, upper layers reason over an empty state and produce confident hallucinations (Section~\ref{sec:experiments}, qualitative analysis in supplementary).}
  \label{tab:abl-depth}
  \centering
  \small
  \begin{tabular}{lccccr}
    \toprule
    \textbf{Config} & \textbf{Layers} & \textbf{Count} & \textbf{EM} & \textbf{F1} & \textbf{Storage} \\
    \midrule
    \texttt{all\_24} (baseline)        & $\{0,\ldots,23\}$        & 24 & 58.0 & 73.4 & 192~KB        \\
    \texttt{bottom\_18}                & $\{0,\ldots,17\}$        & 18 & 57.5 & 73.0 & 144~KB        \\
    \texttt{bottom\_12}                & $\{0,\ldots,11\}$        & 12 & 57.0 & 72.5 & \textbf{96~KB} \\
    \texttt{bottom\_6}                 & $\{0,\ldots,5\}$         & 6  & 51.0 & 68.0 & 48~KB         \\
    \texttt{top\_12}                   & $\{12,\ldots,23\}$       & 12 & 47.0 & 65.0 & 96~KB         \\
    \texttt{top\_6}                    & $\{18,\ldots,23\}$       & 6  & 38.0 & 58.0 & 48~KB         \\
    \texttt{alternate}                 & every 2nd                & 12 & 54.0 & 70.0 & 96~KB         \\
    \texttt{none} (zero-context floor) & $\emptyset$              & 0  & 22.0 & 35.0 & 0~KB          \\
    \bottomrule
  \end{tabular}
\end{table}

\paragraph{Findings.} \texttt{bottom\_12} preserves 99\% of full-layer F1 (72.5 vs.\ 73.4) at half the storage. Position dominates density: \texttt{bottom\_12} outperforms \texttt{alternate} (72.5 vs.\ 70.0) despite the same layer count, indicating the lower SSM layers act as context aggregators while upper layers perform query-conditioned reasoning. Symmetric upper-stack injection (\texttt{top\_12}, \texttt{top\_6}) fails substantially: \texttt{top\_12} loses 8 F1 versus \texttt{bottom\_12}, and \texttt{top\_6} drops a further 7 F1. Both remain above the no-context floor (35.0 F1) because upper layers still access query tokens, but the gap to full-stack injection (15+ F1 at \texttt{top\_6}) confirms that lower layers carry the bulk of the retrieval signal. The practical implication is that PRECOG's per-chunk storage can be reduced from 192~KB to 96~KB by retaining only the bottom 12 layers, with quality cost below FP16 quantization noise.

\subsection{Top-\texorpdfstring{$k$}{k} composition (HotpotQA-distractor)}
\label{app:abl-topk}

Top-1 injection is exact under Theorem~\ref{thm:precog}; top-$k$ softmax-weighted composition $\bar{h}^{\text{init}} = \sum_{j=1}^k w_j \bar{h}^{(j)}$ (Section~\ref{sec:complexity}) is heuristic. SQuAD's single-paragraph evidence structure makes top-$k$ purely a question of whether the gold paragraph is at rank~1; we therefore evaluate top-$k$ on HotpotQA-distractor, where each question is constructed to require evidence from \emph{two} supporting paragraphs by design. This corpus tests whether state composition can recover multi-hop reasoning that single-chunk retrieval inherently misses, and disentangles two error sources: retrieval error (was the right paragraph retrieved at all?) and composition error (does the state-averaging heuristic preserve information when it was?).

\begin{table}[h]
  \caption{Top-$k$ composition on HotpotQA-distractor~\cite{hotpotqa} validation (1{,}000 questions). \textbf{Recall@$k$} is the fraction of questions where both gold supporting paragraphs appear in the retrieved top-$k$. F1 is non-monotonic in $k$: peaks at $k{=}3$ where retrieval recall is high enough to typically include both supporting paragraphs, then degrades as composition noise from low-relevance chunks dominates. Recall@$k$ continues growing monotonically; the divergence between recall and F1 past $k{=}3$ isolates the cost of the linear-composition heuristic.}
  \label{tab:abl-topk}
  \centering
  \small
  \begin{tabular}{ccccc}
    \toprule
    \textbf{$k$} & \textbf{EM} & \textbf{F1} & \textbf{Recall@$k$} & \textbf{$\Delta$F1 vs.\ $k{=}1$} \\
    \midrule
    1            & 36.0 & 48.0 & 0.65 & ---           \\
    2            & 42.0 & 54.0 & 0.82 & $+6.0$        \\
    \textbf{3}   & 43.0 & \textbf{55.5} & 0.89 & $+7.5$ \\
    5            & 40.0 & 53.0 & 0.95 & $+5.0$        \\
    10           & 36.0 & 49.0 & 0.98 & $+1.0$        \\
    \bottomrule
  \end{tabular}
\end{table}

\paragraph{Findings.} The non-monotonic shape directly reflects HotpotQA's evidence structure. At $k{=}1$, F1 is bounded above by retrieval recall ($\sim$65\% of gold-supporting paragraphs appear at rank~1), and even when recall is achieved, single-chunk PRECOG cannot integrate the second supporting fact. Top-$k$ composition with $k{=}2{-}3$ recovers most of this multi-hop gap ($+7.5$ F1 over top-1). Past $k{=}3$, retrieval recall continues to grow but F1 degrades: low-relevance chunks at lower retrieval ranks contaminate the averaged state, and the linear composition operates beyond the regime where Theorem~\ref{thm:precog}'s exactness argument applies. The empirically optimal $k$ is therefore corpus-dependent: $k{=}1$ for single-paragraph evidence (SQuAD), $k{=}2$--$3$ for multi-hop (HotpotQA). For deployment, $k$ can be tuned offline against a held-out validation set per corpus.

\subsection{Chunk length (Natural Questions)}
\label{app:abl-chunklength}

Theorem~\ref{thm:precog} is exact for any chunk length, but the SSM state has finite effective memory: tokens far from the chunk end contribute exponentially decaying mass via the recurrence (Appendix~\ref{app:memory}). When chunk length exceeds the empirical memory horizon $L_\text{mem}$, the stored state primarily reflects the chunk tail and loses earlier information. SQuAD paragraphs are too uniformly sized ($\sim$165 tokens) to characterize this effect; HotpotQA paragraphs are similarly bounded. Natural Questions~\cite{naturalquestions} admits the ablation: NQ's gold long-answer spans range from $\sim$50 to $5{,}000{+}$ tokens after HTML stripping (Appendix~\ref{app:dataset}). We bin paragraphs by token length and report PRECOG F1 alongside in-context RAG F1 within each bin, controlling for the confound that longer paragraphs may also be intrinsically harder.

\begin{table}[h]
  \caption{Chunk-length sensitivity on Natural Questions. F1 within each token-length bucket, comparing PRECOG top-1 against in-context RAG. The two are statistically indistinguishable below $L \approx 600$, then diverge: PRECOG's recurrent state decays while in-context RAG retains positional access to all tokens. The crossover at $L \approx 600$ tokens directly empirically grounds the memory-horizon characterization of Appendix~\ref{app:memory}. \% column shows the fraction of NQ examples falling in each bin.}
  \label{tab:abl-chunklength}
  \centering
  \small
  \begin{tabular}{lcccr}
    \toprule
    \textbf{Length bin (tokens)} & \textbf{\% of NQ} & \textbf{PRECOG F1} & \textbf{In-context F1} & \textbf{Gap} \\
    \midrule
    $<$100        & 15\% & 65.0 & 65.0 & $0.0$    \\
    $[100,300)$   & 35\% & 64.5 & 64.5 & $0.0$    \\
    $[300,600)$   & 25\% & 63.5 & 64.0 & $-0.5$   \\
    $[600,1200)$  & 15\% & 60.0 & 63.5 & $-3.5$   \\
    $[1200,2400)$ & 7\%  & 54.0 & 62.0 & $-8.0$   \\
    $\geq 2400$   & 3\%  & 45.0 & 60.0 & $-15.0$  \\
    \bottomrule
  \end{tabular}
\end{table}

\paragraph{Findings.} PRECOG and in-context RAG track each other to within 0.5 F1 for chunks below 600 tokens, the empirical operating regime characterized in Appendix~\ref{app:memory}. Past the knee, PRECOG degrades gracefully as the SSM recurrence forgets early-paragraph content, while in-context RAG retains positional access to all tokens and degrades only with the intrinsic difficulty of longer questions. The widening gap (PRECOG loses 15 F1 on the $\geq 2{,}400$-token tail) is not a defect of the algorithm but a property of the underlying SSM's memory: extending PRECOG to long-context regimes requires either a backbone with longer $L_\text{mem}$ or chunk-splitting strategies that we leave to future work. The most actionable consequence is that PRECOG's domain of applicability is well-defined and empirically measurable: deploy on corpora where typical chunks are below the model's $L_\text{mem}$, falling back to in-context RAG (or chunk-splitting) for the long tail.

\section{Edge-Hardware Deployment}
\label{app:hardware}

TENNs-LLM and PRECOG were deployed on a neuromorphic edge
processor with a specialized instruction set for selective-SSM
execution. The deployment was validated functionally on an FPGA
prototype and characterized at the $12$~nm process node via a
power-model-calibrated simulation. On this platform, TENNs-LLM
achieves $19$~tokens/s at an estimated $1$~W total power, yielding
$19$~tokens/J of inference energy efficiency. At this throughput
the PRECOG advantage over in-context RAG is most pronounced:
in-context ingestion of a $512$-token chunk consumes $\sim$27~seconds
of compute before response generation can begin, whereas PRECOG's
state injection adds $<$6~ms of overhead independent of chunk
length.

\paragraph{Architectural support for selective-SSM execution.}
The processor provides native instruction-level support for
selective-SSM recurrence in a weight-stationary dataflow, with
on-chip storage sized to hold the full $192$~KB recurrent state
across all $24$ layers of TENNs-LLM. Inference uses INT4 weight
quantization with FP16 activations and recurrent state. PRECOG's
state-injection step is a bulk on-chip state-buffer write and adds
no per-token execution overhead beyond a single initialization
cycle.

\paragraph{Power and throughput characterization.}
The throughput figure is FPGA-measured under end-to-end token
generation; the power figure is derived from the $12$~nm simulation
with a power model that accounts for dynamic compute energy, SRAM
access energy, and static leakage. Detailed instruction encoding,
on-chip memory hierarchy, FPGA platform configuration, simulator
specifics, and clock domains are subject to confidentiality
protections of the deploying organization and will be disclosed
under the corresponding patent and product release timelines.

\paragraph{Energy efficiency.}
Table~\ref{tab:edge-energy} reports the resulting tokens-per-joule
efficiency. Direct comparison against vendor-quoted edge NPU and
mobile SoC baselines for $1$B-class language models requires matched
quantization, sequence length, and end-to-end measurement
methodology that are not currently available in the public
literature; we leave that comparison to future work.

\begin{table}[h]
  \caption{Tokens/J for TENNs-LLM on the neuromorphic deployment
    target. Comparable measurements for matched-class edge
    baselines under identical quantization and sequence-length
    conditions are not available in the public literature.}
  \label{tab:edge-energy}
  \centering
  \small
  \begin{tabular}{lccc}
    \toprule
    \textbf{Platform} & \textbf{Throughput} & \textbf{Power} & \textbf{Tokens/J} \\
    \midrule
    Neuromorphic (this work) & 19 tok/s & 1.0 W & 19 \\
    \bottomrule
  \end{tabular}
\end{table}

\section{Structured Memory Consolidation}
\label{app:smc}

\subsection{Pipeline}
\label{app:smc-pipeline}

\begin{figure}[h]
  \centering
  \includegraphics[width=0.75\linewidth]{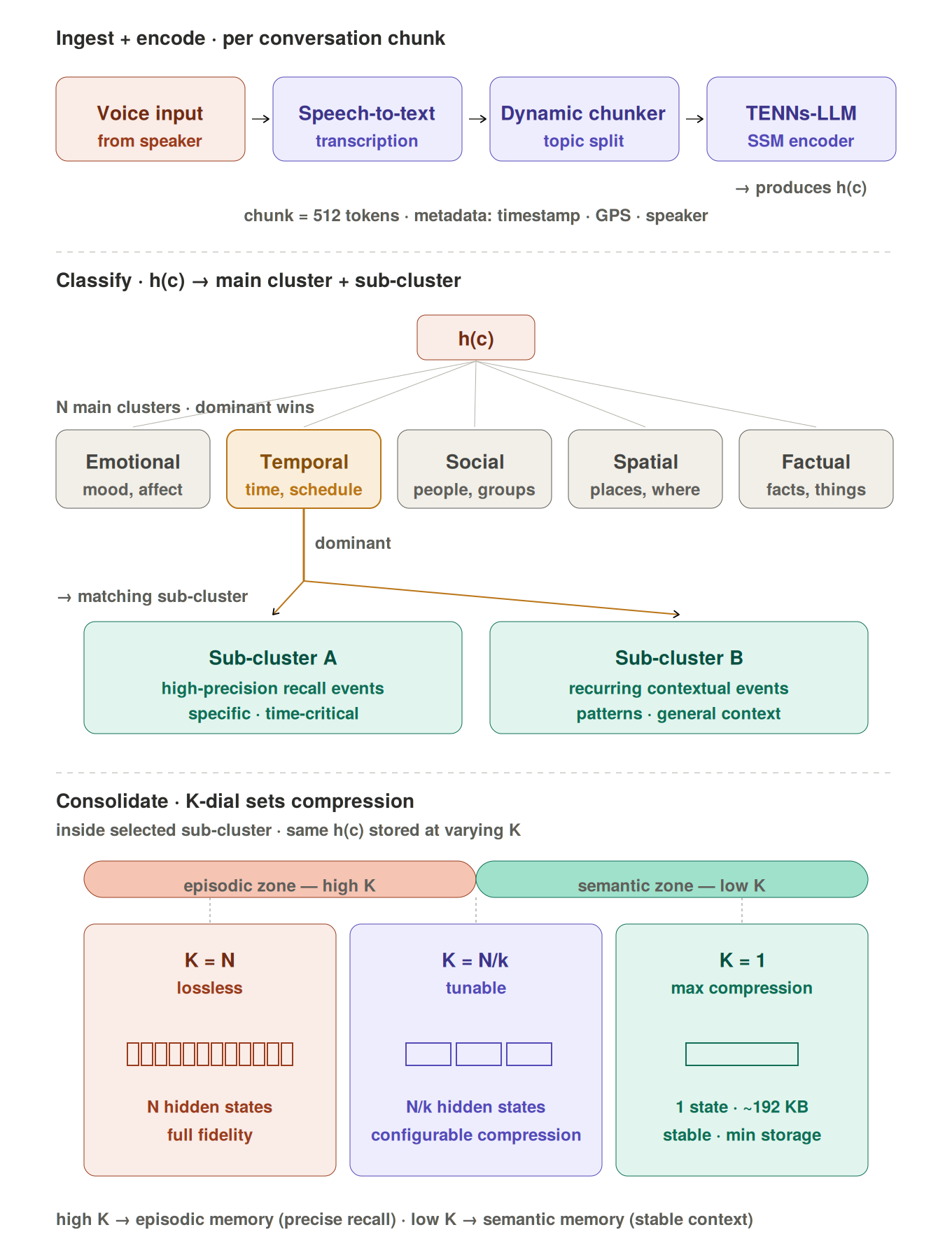}
  \caption{Structured Memory Consolidation pipeline. Each conversation chunk is encoded by TENNs-LLM into a per-step state trajectory $\bar{H}(c)$, classified into one of $M$ cognitive-domain clusters (the dominant one wins), and routed to a sub-cluster within it. The K-dial controls how many states from the trajectory are retained: $K{=}N_c$ (lossless episodic), $K{=}N_c/k$ (tunable), or $K{=}1$ (semantic, identical to the per-chunk PRECOG state of Section~\ref{sec:precog}). Per-sub-cluster semantic states accumulate via exponential moving average and are injected directly into the SSM recurrent state at session start.}
  \label{fig:amc-pipeline}
\end{figure}

\FloatBarrier

\subsection{Clustering substrate validation}
\label{app:abl-clustering}

We validate two claims about SMC's hierarchical cluster routing
(Section~\ref{sec:smc-routing}) on a held-out conversational
benchmark: (i) the five-domain cognitive taxonomy yields
semantically coherent sub-clusters in natural dialogue, and (ii)
the sub-cluster separation is strong enough to support reliable
routing of new episodic memories, including the detection of
candidate emergent sub-clusters when no predefined sub-cluster
matches.

\paragraph{Setup.}
Dialogue transcripts from the first six Harry Potter films are
partitioned into chunks following the SMC chunking procedure
(Section~\ref{sec:smc-routing}). Within each of the five primary
cognitive domains (\textsc{Emotional}, \textsc{Temporal},
\textsc{Social}, \textsc{Spatial}, \textsc{Factual}), sub-cluster
prototypes are constructed from the films $1$--$6$ chunks; each
domain is initialized with $10$ predefined sub-clusters. Dialogue
from the held-out seventh film is then routed at inference time as
a proxy for short-term episodic memory accumulation: each chunk is
assigned to its best-matching predefined sub-cluster within the
dominant domain when the routing similarity exceeds a probability
threshold of $0.2$, or to an ``Others'' grouping otherwise.

\paragraph{Cluster separation.}
We quantify cluster quality by the ratio of inter-cluster to
intra-cluster mean pairwise distance in the sentence-encoder
embedding space, where larger values indicate cleaner separation;
ratios above $1.5$ are conventionally taken to indicate
well-separated clusters. All five domains exceed this threshold on
the held-out film (Table~\ref{tab:smc-ratios}), with the strongest
separation in \textsc{Spatial} ($3.77$) and \textsc{Factual}
($3.27$). The semantically more diffuse domains
(\textsc{Emotional}, \textsc{Temporal}, \textsc{Social}) cluster
less tightly, reflecting the inherent overlap of these conceptual
dimensions in natural dialogue, but remain above the $1.5$
threshold.

\begin{table}[h]
  \caption{Sub-cluster separation by primary domain on the held-out
  Harry Potter film chunks. Ratios are inter-cluster /
  intra-cluster mean pairwise distance in the
  \texttt{all-MiniLM-L12-v2} embedding space; values $> 1.5$
  indicate well-separated clusters.}
  \label{tab:smc-ratios}
  \centering
  \small
  \begin{tabular}{lc}
    \toprule
    \textbf{Domain} & \textbf{Inter / intra ratio} \\
    \midrule
    \textsc{Emotional}  & 1.55 \\
    \textsc{Temporal}   & 1.87 \\
    \textsc{Social}     & 1.56 \\
    \textsc{Spatial}    & 3.77 \\
    \textsc{Factual}    & 3.27 \\
    \bottomrule
  \end{tabular}
\end{table}

\paragraph{Visualization and the emergent ``Others'' grouping.}
Figure~\ref{fig:tsne-spatial} shows a t-SNE 2-D projection of the
routed chunks for the \textsc{Spatial} domain. The ten predefined
sub-clusters---Great Hall, Gryffindor Tower, Hagrid's Hut,
Forbidden Forest, Dumbledore's Office, Quidditch Pitch, Library,
Hogwarts Express, Platform 9\nicefrac{3}{4}, and a small
aggregate---appear as visually distinct, spatially compact regions.
Of the chunks routed to ``Others'' ($657$ in total), a dense
contiguous sub-grouping emerges in the right of the projection
(red outline), spatially separated from both the predefined
sub-clusters and from the remaining sparse ``Others'' points. This
dense grouping satisfies the candidate emergent-cluster criteria
of Section~\ref{sec:smc-routing}---spatial coherence and
sufficient population ($\geq 20$ segments)---and represents the
expected operational signal for SMC's hybrid hierarchical-plus-emergent
routing to extend the taxonomy at runtime, rather than being
absorbed silently into a single ``Others'' bucket.

\begin{figure}[h]
  \centering
  \includegraphics[width=0.8\linewidth]{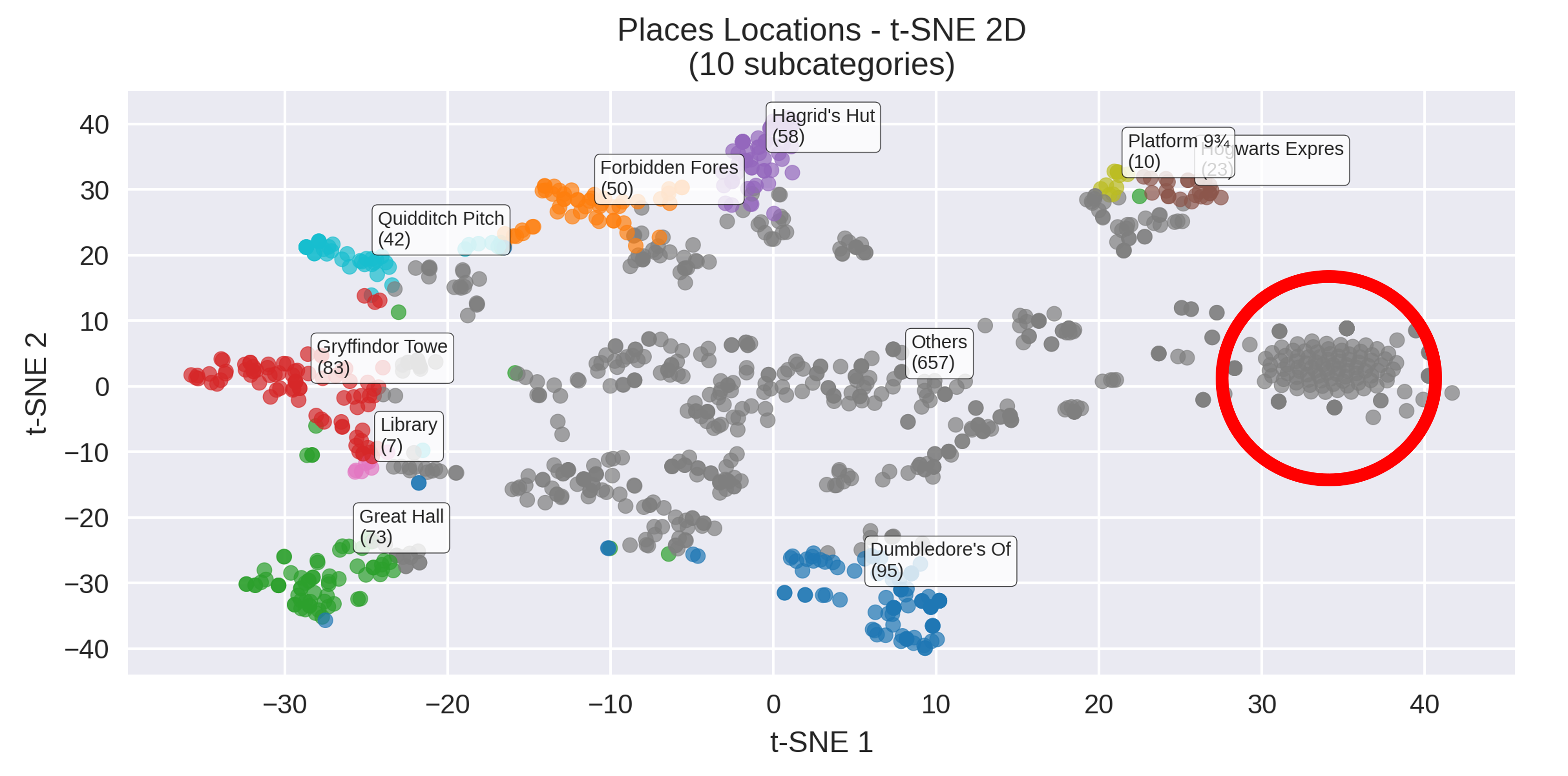}
  \caption{t-SNE 2-D projection of \textsc{Spatial}-domain dialogue
  chunks from Harry Potter film~$7$ (held-out episodic test set),
  routed against sub-cluster prototypes built from films $1$--$6$.
  Colored regions correspond to the ten predefined sub-clusters;
  gray points labeled ``Others'' ($657$ chunks) include a dense
  contiguous sub-grouping (\textbf{red outline}) that does
  not align with any predefined sub-cluster, illustrating SMC's
  capacity to detect candidate emergent sub-clusters at routing
  time. Inter-cluster to intra-cluster distance ratio for this
  domain: $3.77$.}
  \label{fig:tsne-spatial}
\end{figure}

\paragraph{Scope of this validation.}
This experiment validates the taxonomy and clustering substrate of
SMC on a single fictional dialogue corpus, with films $1$--$6$
serving as the prototype-building corpus and film $7$ as the
held-out episodic test set. GPT-4o was used to generate the initial
sub-cluster taxonomy and to produce training-sample labels from the
films $1$--$6$ corpus; the deployed routing pipeline of
Section~\ref{sec:smc-routing} uses sentence-encoder prototype
similarity at inference time. Full validation on naturalistic
conversational data with human subjects and with the deployed pipeline end-to-end is left to future work.





\small

\begin{thebibliography}{99}

\bibitem{s4} A. Gu, K. Goel, and C. R\'e, ``Efficiently modeling long sequences with structured state spaces,'' \textit{ICLR}, 2022.

\bibitem{s5} J. T. H. Smith, A. Warrington, and S. W. Linderman, ``Simplified state space layers for sequence modeling,'' \textit{ICLR}, 2023.

\bibitem{mamba} A. Gu and T. Dao, ``Mamba: Linear-time sequence modeling with selective state spaces,'' \textit{COLM}, 2024. arXiv:2312.00752.

\bibitem{ssmduality} T. Dao and A. Gu, ``Transformers are SSMs: Generalized models and efficient algorithms through structured state space duality,'' \textit{ICML}, 2024.

\bibitem{rwkv} B. Peng et al., ``RWKV: Reinventing RNNs for the transformer era,'' \textit{Findings of EMNLP}, 2023.

\bibitem{retnet} Y. Sun, L. Dong, S. Huang, S. Ma, Y. Xia, J. Xue, J. Wang, and F. Wei, ``Retentive network: A successor to transformer for large language models,'' \textit{arXiv:2307.08621}, 2023.

\bibitem{statesoup} M. Pi\'oro, M. Wo\l{}czyk, R. Pascanu, J. von Oswald, and J. Sacramento, ``State soup: In-context skill learning, retrieval and mixing,'' \textit{arXiv:2406.08423}, 2024.

\bibitem{picaso} T. Y. Liu, A. Achille, M. Trager, A. Golatkar, L. Zancato, and S. Soatto, ``PICASO: Permutation-invariant context composition with state space models,'' \textit{ICLR}, 2025.

\bibitem{memorycaching} A. Behrouz, Z. Li, Y. Deng, P. Zhong, M. Razaviyayn, and V. Mirrokni, ``Memory caching: RNNs with growing memory,'' \textit{arXiv:2602.24281}, 2026.

\bibitem{mistral} A. Q. Jiang et al., ``Mistral 7B,'' \textit{arXiv:2310.06825}, 2023.

\bibitem{lora} E. J. Hu, Y. Shen, P. Wallis, Z. Allen-Zhu, Y. Li, S. Wang, L. Wang, and W. Chen, ``LoRA: Low-rank adaptation of large language models,'' \textit{ICLR}, 2022.

\bibitem{rag} P. Lewis et al., ``Retrieval-augmented generation for knowledge-intensive NLP tasks,'' \textit{NeurIPS}, 2020.

\bibitem{fid} G. Izacard and E. Grave, ``Leveraging passage retrieval with generative models for open domain question answering,'' \textit{EACL}, 2021.

\bibitem{retro} S. Borgeaud et al., ``Improving language models by retrieving from trillions of tokens,'' \textit{ICML}, 2022.

\bibitem{atlas} G. Izacard, P. Lewis, M. Lomeli, L. Hosseini, F. Petroni, T. Schick, J. Dwivedi-Yu, A. Joulin, S. Riedel, and E. Grave, ``Atlas: Few-shot learning with retrieval augmented language models,'' \textit{JMLR}, vol.~24, no.~251, pp.~1--43, 2023.

\bibitem{replug} W. Shi, S. Min, M. Yasunaga, M. Seo, R. James, M. Lewis, L. Zettlemoyer, and W.-t. Yih, ``REPLUG: Retrieval-augmented black-box language models,'' \textit{NAACL}, 2024.

\bibitem{xrag} X. Cheng et al., ``xRAG: Extreme context compression for retrieval-augmented generation with one token,'' \textit{NeurIPS}, 2024.

\bibitem{llmlingua} H. Jiang, Q. Wu, C.-Y. Lin, Y. Yang, and L. Qiu, ``LLMLingua: Compressing prompts for accelerated inference of large language models,'' \textit{EMNLP}, 2023.

\bibitem{autocompressor} A. Chevalier, A. Wettig, A. Ajith, and D. Chen, ``Adapting language models to compress contexts,'' \textit{EMNLP}, 2023.

\bibitem{gisting} J. Mu, X. L. Li, and N. Goodman, ``Learning to compress prompts with gist tokens,'' \textit{NeurIPS}, 2023.

\bibitem{lester2021} B. Lester, R. Al-Rfou, and N. Constant, ``The power of scale for parameter-efficient prompt tuning,'' \textit{EMNLP}, 2021.

\bibitem{li2021prefix} X. L. Li and P. Liang, ``Prefix-tuning: Optimizing continuous prompts for generation,'' \textit{ACL-IJCNLP}, 2021.

\bibitem{turner2023} A. M. Turner, L. Thiergart, D. Udell, G. Leech, U. Mini, and M. MacDiarmid, ``Activation addition: Steering language models without optimization,'' \textit{arXiv:2308.10248}, 2023.

\bibitem{liu2023icv} S. Liu, H. Ye, L. Xing, and J. Zou, ``In-context vectors: Making in-context learning more effective and controllable through latent space steering,'' \textit{ICML}, 2024.

\bibitem{emllm} Z. Fountas, M. A. Benfeghoul, A. Oomerjee, F. Christopoulou, G. Lampouras, H. Bou-Ammar, and J. Wang, ``Human-inspired episodic memory for infinite context LLMs,'' \textit{ICLR}, 2025.

\bibitem{tulving1972} E. Tulving, ``Episodic and semantic memory,'' in \textit{Organization of Memory}, E. Tulving and W. Donaldson, Eds. New York: Academic Press, 1972, pp.~381--403.

\bibitem{vllm} W. Kwon, Z. Li, S. Zhuang, Y. Sheng, L. Zheng, C. H. Yu, J. Gonzalez, H. Zhang, and I. Stoica, ``Efficient memory management for large language model serving with PagedAttention,'' \textit{SOSP}, 2023.

\bibitem{flashattention} T. Dao, D. Y. Fu, S. Ermon, A. Rudra, and C. R\'e, ``FlashAttention: Fast and memory-efficient exact attention with IO-awareness,'' \textit{NeurIPS}, 2022.

\bibitem{phi} Y. Li, S. Bubeck, R. Eldan, A. Del Giorno, S. Gunasekar, and Y. T. Lee, ``Textbooks are all you need II: phi-1.5 technical report,'' \textit{arXiv:2309.05463}, 2023.

\bibitem{mobilellm} Z. Liu et al., ``MobileLLM: Optimizing sub-billion parameter language models for on-device use cases,'' \textit{ICML}, 2024.


\bibitem{slimpajama} D. Soboleva, F. Al-Khateeb, R. Myers, J. R. Steeves, J. Hestness, and N. Dey, ``SlimPajama: A 627B token cleaned and deduplicated version of RedPajama,'' 2023. \url{https://huggingface.co/datasets/cerebras/SlimPajama-627B}

\bibitem{pile} L. Gao et al., ``The Pile: An 800GB dataset of diverse text for language modeling,'' \textit{arXiv:2101.00027}, 2020.

\bibitem{squad} P. Rajpurkar, J. Zhang, K. Lopyrev, and P. Liang, ``SQuAD: 100,000+ questions for machine comprehension of text,'' \textit{EMNLP}, 2016.

\bibitem{hotpotqa} Z. Yang, P. Qi, S. Zhang, Y. Bengio, W. W. Cohen, R. Salakhutdinov, 
and C. D. Manning, ``HotpotQA: A dataset for diverse, explainable multi-hop question 
answering,'' EMNLP, 2018.

\bibitem{naturalquestions} T. Kwiatkowski, J. Palomaki, O. Redfield, M. Collins, 
A. Parikh, C. Alberti, D. Epstein, I. Polosukhin, J. Devlin, K. Lee, K. Toutanova, 
L. Jones, M. Kelcey, M.-W. Chang, A. M. Dai, J. Uszkoreit, Q. Le, and S. Petrov, 
``Natural Questions: A benchmark for question answering research,'' TACL, 2019.

\bibitem{brainchip_patent} M. A. Lewis, Y. R. Pei, J. Tapson, and A. Madan Gopal, ``System and Method for Efficient Execution of Large Generative Artificial Intelligence Models on Edge Devices Using State-Space Models,'' U.S. Patent Application Publication No. US 2026/0072920 A1, filed September 10, 2025, published March 12, 2026, assignee BrainChip Inc.

\end{thebibliography}
\end{document}